%% file: arxiv.tex
\documentclass{article}

\usepackage[preprint]{neurips_2026}

\usepackage[utf8]{inputenc}
\usepackage[T1]{fontenc}
\usepackage{graphicx}
\usepackage{needspace}
\usepackage{placeins}
\usepackage{wrapfig}
\usepackage{hyperref}
\usepackage{url}
\usepackage{booktabs}
\usepackage{array}
\usepackage{amsfonts}
\usepackage{nicefrac}
\usepackage{microtype}
\usepackage{xcolor}
\usepackage{algorithm}
\usepackage{algpseudocode}
\usepackage{amsmath, amssymb}
\usepackage{amsthm}
\usepackage{todonotes}
\usepackage{cleveref}
\usepackage[most]{tcolorbox}

\DeclareMathOperator{\polar}{polar}
\DeclareMathOperator{\FreshGNS}{FreshGNS}
\DeclareMathOperator{\CacheGNS}{CacheGNS}

\tcbset{
  theorembox/.style={
    enhanced,
    breakable,
    colback=gray!5,
    colframe=gray!30,
    boxrule=0.3pt,
    arc=1.5pt,
    left=5pt,
    right=5pt,
    top=5pt,
    bottom=5pt,
    before skip=6pt,
    after skip=6pt,
  }
}

\newtheorem{theorem}{Theorem}[section]
\newtheorem{proposition}[theorem]{Proposition}
\newtheorem{lemma}[theorem]{Lemma}
\theoremstyle{remark}
\newtheorem{remark}[theorem]{Remark}

\tcolorboxenvironment{theorem}{theorembox}
\tcolorboxenvironment{proposition}{theorembox}
\tcolorboxenvironment{lemma}{theorembox}

\title{CacheMuon: Using Temporal Preconditioning To Approximate Polar Factor}

\author{%
  Bishnu Dev \\
  Mohamed bin Zayed University\\
  of Artificial Intelligence\\
  \texttt{bishnu.dev@mbzuai.ac.ae}
  \And
  Sushil Bohara \\
  Mohamed bin Zayed University\\
  of Artificial Intelligence\\
  \texttt{sushil.bohara@mbzuai.ac.ae}
  \AND
  Martin Tak\'a\v{c} \\
  Mohamed bin Zayed University\\
  of Artificial Intelligence\\
  \texttt{martin.takac@mbzuai.ac.ae}
  \And
  Samuel Horv\'ath \\
  Mohamed bin Zayed University\\
  of Artificial Intelligence\\
  \texttt{samuel.horvath@mbzuai.ac.ae}
}

\begin{document}

\maketitle

\begin{abstract}
  Muon is an optimizer that computes updates using the polar factor of the momentum matrix and has shown strong empirical performance across a range of training settings. A key component of Muon is the Newton--Schulz iteration used to compute this polar factor. Although this avoids the cost of an exact singular value decomposition, it remains expensive in practice because it is applied at every optimization step. At the same time, the momentum matrix changes smoothly over training, suggesting strong temporal correlation in the corresponding polar factors. In this paper, we exploit this structure and propose CacheMuon, a temporal preconditioning method that reuses information from previous optimization steps to approximate the polar factor at the current step. This reduces redundant orthogonalization computation across iterations. We analyze CacheMuon as an inexact Muon update, with error controlled by fresh-solver error and cache staleness. Empirically, CacheMuon provides a controllable quality--efficiency frontier: conservative thresholds closely match fresh Muon on language-model and vision training while reducing orthogonalization FLOPs, whereas more aggressive thresholds yield larger arithmetic savings at the cost of modest validation-quality degradation.
\end{abstract}

\input{sections/introduction}
\input{sections/related_work.tex}
\input{sections/method}
\input{sections/convergence_analysis}
\input{sections/complexity_implementation.tex}
\input{sections/numerical_experiments.tex}
\input{sections/discussion_and_limitation.tex}
\input{sections/conclusion.tex}

\bibliographystyle{plainnat}
\bibliography{references}

\clearpage
\appendix
\input{sections/appendix1}
\clearpage
\input{sections/appendix2.tex}
\clearpage
\input{sections/appendix3.tex}

\end{document}

%% file: sections/introduction.tex
\section{Introduction}
\label{sec:introduction}
Muon is a geometry-aware optimizer for matrix-valued hidden-layer parameters that replaces standard coordinate-wise updates such as Adam and AdamW with an orthogonalized momentum direction \citep{adam,adamW,muon}. Given a matrix momentum update, Muon applies an approximate polar-factor computation and uses the resulting semi-orthogonal matrix as the descent direction. This matrix-aware update has shown strong empirical performance, and recent work has demonstrated that Muon remains effective beyond the small-model regime and can scale to large language model training \citep{muon_scalable}.

The main practical bottleneck in Muon is the orthogonalization step. In modern implementations, the exact polar factor is approximated using a small number of Newton--Schulz iterations, which avoid an explicit singular value decomposition but still require several matrix multiplications at each optimizer step \citep{muon}. As a result, a nontrivial portion of Muon's cost is spent repeatedly solving closely related matrix orthogonalization problems throughout training. Recent work has focused on improving the efficiency of a \emph{single} orthogonalization call, for example through optimized low-accuracy matrix-sign iterations such as PolarExpress and through hardware-aware reformulations such as Gram Newton--Schulz \citep{polar_express,gram_newton_schulz}. These advances make each call cheaper, but they still recompute the orthogonalizing transform from scratch at every step. Moreover, because Muon's update relies on a matrix-level orthogonalization that couples all entries of a parameter matrix, it is less amenable to efficient parallelization than coordinate-wise optimizers. As a result, the orthogonalization cost is both repeated and difficult to distribute, motivating approaches that reduce the number of such computations across training.

In this work, we ask whether Muon's orthogonalization can also be improved \emph{across time}. Our starting point is the observation that adjacent momentum matrices often evolve smoothly during training, especially once optimization has entered a stable regime. If the current momentum matrix is close to the previous one, then the transformation computed at the previous step should remain informative for the current polar-factor approximation.

Motivated by this idea, we propose \emph{CacheMuon}, a Muon variant that reuses a previously computed orthogonalization transform across optimization steps. At a high level, the method maintains a cached left transform produced by Gram Newton--Schulz at a previous step and applies it to the current normalized momentum matrix. To determine whether the cached transform remains sufficiently accurate, we introduce a residual-based reuse trigger that measures the orthogonality error of the candidate update and selectively refreshes the cache by running a fresh Gram Newton--Schulz solve. Unlike prior improvements that focus only on the inner polynomial iteration or the matrix-multiplication kernel, our method introduces \emph{temporal reuse} into Muon's orthogonalization pipeline. Its purpose is to reduce the amount of orthogonalization work performed over training by avoiding unnecessary full correction passes.

Temporal reuse changes the update direction returned by Muon, so the main question is not only computational but also algorithmic: does Muon still converge when the orthogonalization is produced by a cached transform rather than a fresh solve at every step? We answer this by viewing CacheMuon as an \emph{inexact Muon} method. Concretely, we show that the cached direction differs from the exact Muon direction by an additive approximation error controlled by two quantities: the error of the fresh orthogonalization routine and the staleness of the cached Gram matrix relative to the current one. This allows the cached update to be analyzed within the recent theory of practical and inexact Muon updates, yielding a stochastic nonconvex convergence guarantee when the fresh solver error and cache staleness remain controlled \citep{muon_ns_convergence,muon_inexact,norm_constrained_lmo}.

Empirically, we evaluate CacheMuon along three complementary axes: end-to-end training quality, reductions in orthogonalization FLOPs, and the effectiveness of the residual-based reuse trigger as a proxy for cached-update fidelity. We show that low residual corresponds to strong agreement between cached and fresh orthogonalization directions, making the trigger an effective proxy for cached-update fidelity and allowing it to selectively reuse updates that remain close to fresh Muon direction. Across language-model and vision experiments, conservative cache reuse closely tracks fresh-Muon training behavior while reducing orthogonalization FLOPs, whereas more aggressive reuse yields larger arithmetic savings at the cost of modest validation-quality degradation. Our contributions are the following:
\begin{itemize}
    \item We introduce \emph{CacheMuon}, a Muon variant that exploits temporal smoothness in the momentum matrix by reusing a cached orthogonalization transform across optimization steps.
    \item We show that CacheMuon can be analyzed as an \emph{inexact Muon} method and derive a stochastic convergence guarantee under controlled fresh-solver error and cache staleness.
    \item We propose a residual-based reuse trigger that determines when a cached transform remains sufficiently accurate and when a fresh Muon direction should be recomputed.
    \item We empirically demonstrate across language-model and vision workloads that temporal reuse can substantially reduce orthogonalization FLOPs while closely tracking fresh Muon's training behavior under conservative reuse, yielding a controllable quality--efficiency frontier.
\end{itemize}

%% file: sections/related_work.tex
\section{Related Work}

\paragraph{Muon and efficient pretraining.}
Muon updates matrix-valued hidden-layer parameters by replacing a coordinate-wise adaptive update
with an orthogonalized momentum direction computed through an approximate polar factor
\citep{muon}. Its empirical performance has motivated recent work on using Muon for large-scale
language-model training \citep{muon_scalable, muon_pretraining_efficiency}. This scaling regime
makes the cost of orthogonalization important: although Newton--Schulz iterations avoid an exact
SVD, they still require multiple matrix multiplications at every optimizer step. CacheMuon is
motivated by this practical bottleneck. Rather than changing the outer optimizer, it targets the
repeated polar-factor computation inside Muon.

\paragraph{Faster polar-factor and Newton--Schulz computation.}
The closest algorithmic work to ours improves the cost or quality of a single fresh
orthogonalization call. PolarExpress designs optimized low-accuracy matrix-sign iterations for
Muon \citep{polar_express}, while Gram Newton--Schulz rewrites the computation around the
smaller Gram matrix and hardware-efficient symmetric operations \citep{gram_newton_schulz}.
Other recent approaches also accelerate the Newton--Schulz orthogonalization itself: Chebyshev-type
polynomial schemes choose improved iteration polynomials for approximate orthogonalization
\citep{chebyshev_ns}, and Turbo-Muon uses a preconditioning procedure to accelerate
Newton--Schulz convergence and reduce the number of iterations required in practice
\citep{turbo_muon}. These methods are complementary to CacheMuon. They make a fresh
orthogonalization call cheaper; CacheMuon reduces how often a full orthogonalization is
computed by reusing information across nearby optimization steps.

\paragraph{Approximate and inexact Muon.}
Practical Muon implementations use a small finite number of Newton--Schulz iterations rather than
the exact SVD polar factor. Recent theory has begun to close the gap between idealized Muon and
these approximate implementations. \citet{muon_ns_convergence} analyze Muon with finite-step Newton--Schulz
orthogonalization, while \citet{muon_inexact} study inexact Muon updates
under an additive approximation-error model. Earlier and complementary
analyses interpret Muon as a spectral-norm steepest-descent method or as a non-Euclidean
trust-region method \citep{muon_convergence_note, gradient_orthogonalization_trust_region}.
Our analysis follows the practical/inexact viewpoint: the cached update differs from fresh Muon
through an additive error controlled by the fresh solver error and the staleness of the cached transform.

\paragraph{Adaptive, preconditioned, and large-scale Muon variants.}
A rapidly growing line of work modifies Muon's update rule to improve robustness, adaptivity, or
large-scale throughput. AdaMuon, AdaGO, NorMuon, NAMO, and Muon$^2$ combine orthogonalized
updates with adaptive stepsizes, second-moment information, or neuron-wise normalization
\citep{adamuon, adago, normuon, namo, muon2}. MuonClip was introduced in the Kimi K2 technical
report to stabilize very large-scale MoE training with a Muon-style optimizer \citep{kimi_k2}.
Distributed and tensorized variants address the systems cost of orthogonalized updates at scale:
Dion replaces Newton--Schulz with amortized power iteration for sharded training
\citep{dion}, MuonBP uses block-periodic orthogonalization to reduce communication overhead
\citep{muonbp}, and TEON generalizes layer-wise Muon to tensorized orthonormalization across
structured gradients \citep{teon}. CacheMuon is different from these variants: it does not introduce
adaptive rescaling, change the layer aggregation rule, or target distributed sharding directly. It keeps
the Muon update geometry fixed and amortizes the orthogonalization oracle over time.

\paragraph{Geometry-aware and LMO-based optimization.}
Muon is also connected to a broader family of layer-wise geometry-aware methods. Some works
explain Muon through spectral-norm constraints or stochastic Frank--Wolfe perspectives
\citep{muon_spectral_constraints, lions_muons}. Scion formulates deep-learning optimization
through norm-constrained linear minimization oracles \citep{norm_constrained_lmo}, and Gluon
connects Muon and Scion within a common layer-wise LMO-based framework \citep{gluon}.
Non-Euclidean formulations further study how different norm choices, aggregation rules, and
normalizations lead to Muon variants \citep{muon_nongeuclidean_variants}. These works clarify
the geometry of the update direction. CacheMuon is orthogonal to them: it assumes the same
Muon-style polar direction and asks whether the computation of that direction can be reused across
successive optimizer steps.

%% file: sections/method.tex
\section{Method}

\label{sec:method}

Muon updates matrix-valued parameters by orthogonalizing a momentum matrix before applying the step \citep{muon}. Let $W_t \in \mathbb{R}^{n\times m}$ be a matrix parameter and let
\[
g_t = \nabla f(W_t;\xi_t)
\]
denote the stochastic gradient at iteration $t$. Throughout, we use the momentum recursion
\[
M_t = (1-\beta) g_t + \beta M_{t-1},
\]
where $\beta \in [0,1)$. If $M_t = U_t \Sigma_t V_t^\top$ is the reduced singular value decomposition of $M_t$, then the exact Muon direction is its polar factor
\[
D_t = \polar(M_t) = U_t V_t^\top.
\]
Computing $D_t$ exactly by singular value decomposition at every training step is too expensive, so Muon instead approximates the polar factor using a small number of Newton--Schulz (\Cref{alg:standard-ns}) iterations \citep{muon}. These iterations preserve the singular vectors of the input while driving the nonzero singular values toward one.

\subsection{Gram Newton--Schulz}
In this work, we build on Gram Newton--Schulz which reformulates the standard Newton--Schulz (\Cref{alg:standard-ns}) used in original Muon in terms of the smaller symmetric Gram matrix \citep{gram_newton_schulz,polar_express,muon}. Let
\[
X_t = \frac{M_t}{\|M_t\|_F + \varepsilon}, \qquad
X_t^{(0)} = X_t, \qquad
R_t^{(0)} = X_t X_t^\top, \qquad
Q_t^{(0)} = I.
\]
At iteration $k$, define
\[
Z_t^{(k)} = b_k R_t^{(k-1)} + c_k \bigl(R_t^{(k-1)}\bigr)^2,
\qquad
P_t^{(k)} = a_k I + Z_t^{(k)}.
\]
The iteration updates
\[
Q_t^{(k)} = Q_t^{(k-1)} P_t^{(k)},
\qquad
R_t^{(k)} = P_t^{(k)} R_t^{(k-1)} P_t^{(k)},
\]
and returns
\[
X_t^{(K)} = Q_t^{(K)} X_t^{(0)}.
\]

In exact arithmetic, this Gram formulation is equivalent to standard Newton--Schulz, but in low precision, it can become unstable. We therefore use the stabilized Gram Newton--Schulz (\Cref{alg:gns}) variant with restarts at selected inner iterations $k \in \mathcal S \subseteq \{1,\dots,K-1\}$ \citep{gram_newton_schulz}: at such a restart index, set $X_t \gets Q_t^{(k)} X_t$, rebuild $R_t = X_t X_t^\top$, reset $Q_t \gets I$, and continue the local recurrence.

\subsection{Cache Gram Newton--Schulz}

Both standard Newton--Schulz and Gram Newton--Schulz recompute an orthogonalizing transform from scratch at every optimization step. While this avoids explicit singular value decompositions, it still incurs a nontrivial cost due to repeated matrix multiplications. At the same time, the momentum matrix $M_t$ typically evolves smoothly over training, so the corresponding orthogonalization problems at adjacent steps are closely related. This suggests that recomputing the transform independently at each step may be redundant.

A key property of stabilized Gram Newton--Schulz is that it can be modified to produce not only an orthogonalized matrix, but also a left transform $Q_t$ such that $\polar(X_t) \approx \tilde D_t = Q_t X_t$. In particular, at every restart, we accumulate the left transform computed on the current segment before rebuilding the Gram matrix and resetting the working factor to \(I\). We call this resulting algorithm Fresh Gram Newton--Schulz (FreshGNS, \Cref{alg:fresh-gns}). The transform $Q_t$ can be interpreted as an approximate orthogonalizing operator for the current momentum matrix. Consequently, when the momentum matrix changes gradually over training, the orthogonalizing operator computed at one step should continue to approximately orthogonalize nearby iterates. Our approach is to cache this operator and reuse it across iterations, thereby amortizing the cost of orthogonalization over time.

Let $r_t \in \{0,1\}$ denote whether step $t$ is a refresh step, with $r_1 = 1$, and define the anchor index
\[
a_t := \max\{s \le t : r_s = 1\}.
\]
If $r_t = 1$, we run Fresh Gram Newton--Schulz (FreshGNS, \Cref{alg:fresh-gns}) and obtain
\[
\tilde D_t, Q_t = \FreshGNS(M_t),
\qquad
\hat D_t = \tilde D_t.
\]
If $r_t = 0$, we reuse the cached transform $Q_{a_t}$ and compute
\[
X_t = \frac{M_t}{\|M_t\|_F + \varepsilon},
\qquad
\hat D_t = Q_{a_t} X_t.
\]

We call the resulting method Cache Gram Newton--Schulz (CacheGNS, \Cref{alg:cache-gns}). The algorithm maintains a cached orthogonalizing operator that is refreshed using FreshGNS on refresh steps ($r_t = 1$) and reused on intermediate optimization steps ($r_t = 0$) to approximately orthogonalize the momentum matrix. In practice, the refresh decisions $r_t$ are determined by a residual-based reuse heuristic described later in \Cref{sec:complexity_implementation}.

\subsection{CacheMuon}
CacheMuon changes only the orthogonalization step of Muon: instead of recomputing the polar factor of the momentum matrix at every iteration, it reuses the most recent cache transform to approximate the polar factor whenever the current step is not refreshed. The resulting optimizer is given by
\[
g_t = \nabla f(W_t;\xi_t),
\qquad
M_t = (1-\beta) g_t + \beta M_{t-1},
\]
\[
(\hat D_t, Q_t)
=
\CacheGNS(M_t),
\]
\[
W_{t+1} = W_t - \eta \hat D_t.
\]
where
\[
\CacheGNS(M_t)
=
\begin{cases}
\FreshGNS(M_t), & r_t = 1,\\[4pt]
(Q_{a_t}X_t,\; Q_{a_t}), & r_t = 0.
\end{cases}
\]

When $r_t = 1$ for every step, CacheMuon reduces to Muon. When refreshes are less frequent, the method amortizes orthogonalization across time by reusing the previously computed cache.

%% file: sections/convergence_analysis.tex
\section{Convergence Analysis}

\label{sec:convergence}

Using the notation from \Cref{sec:method}, let
\[
X_t := \frac{M_t}{\|M_t\|_F + \varepsilon},
\qquad
A_t := X_t X_t^\top
\]
denote the normalized momentum matrix and its Gram matrix, respectively. Let $\mathcal Q(A)$ denote the cumulative left transform returned by the Fresh Gram Newton--Schulz (\Cref{alg:fresh-gns}) solver when applied with Gram matrix $A$ under the fixed coefficients and restart schedule from \Cref{sec:method}. Let
\[
F(W) := \mathbb E_{\xi}[f(W;\xi)]
\]
denote the expected objective. We write
\[
\tilde D_t := \mathcal Q(A_t) X_t
\]
for the direction obtained by running a Fresh Gram Newton--Schulz solve at step $t$, and
\[
\hat D_t := \mathcal Q(A_{a_t}) X_t = Q_{a_t} X_t
\]
for the cached direction actually used by the optimizer. Finally, define the momentum-tracking error
\[
e_t := \|M_t - \nabla F(W_t)\|_{\mathrm{tr}}.
\]

\paragraph{Assumptions.}
We assume the following throughout.

\textbf{(A1)} $F$ is $L$-smooth in the spectral/nuclear geometry:
\[
F(Y) \le F(X) + \langle \nabla F(X), Y-X\rangle + \frac{L}{2}\|Y-X\|_{\mathrm{sp}}^2,
\]
and
\[
\|\nabla F(Y)-\nabla F(X)\|_{\mathrm{tr}} \le L\|Y-X\|_{\mathrm{sp}}.
\]
\textbf{(A2)} The stochastic gradient is unbiased with bounded variance:
\[
g_t = \nabla F(W_t)+\zeta_t,
\qquad
\mathbb E[\zeta_t\mid \mathcal F_t]=0,
\qquad
\mathbb E[\|\zeta_t\|_F^2\mid \mathcal F_t]\le \sigma^2.
\]
\textbf{(A3)} The Fresh Gram Newton--Schulz solver has uniform approximation error:
\[
\|\tilde D_t-D_t\|_{\mathrm{sp}} \le \varepsilon_{\mathrm{ref}}
\qquad \text{for all } t.
\]
\textbf{(A4)} The cached Gram matrix remains within bounded staleness of the current one:
\[
\tau_t := \|A_t-A_{a_t}\|_{\mathrm{sp}},
\qquad
\bar\tau := \sup_t \tau_t < \infty.
\]

Assumptions \textbf{(A1)} \& \textbf{(A2)} are standard in recent stochastic analyses of Muon and inexact Muon updates \citep{norm_constrained_lmo,gluon,muon_inexact,fedmuon}. The additional requirements, \textbf{(A3)} \& \textbf{(A4)}, arise from our setting: \textbf{(A3)} aligns with prior analysis of finite-step orthogonalization error (\citep{polar_express}), whereas \textbf{(A4)} is unique to this work, capturing temporal reuse via bounded Gram staleness.

The next proposition shows that the effect of caching can be summarized by a single inexactness term.

\begin{proposition}[Cache reuse induces controlled inexactness]
\label{prop:cache-inexact}
Let
\[
\mathcal D := \{A\in\mathbb R^{n\times n}: A=A^\top,\ A\succeq 0,\ \|A\|_{\mathrm{sp}}\le 1\}.
\]
Then $A_t\in\mathcal D$ for all $t$. Moreover, for fixed coefficients and a fixed restart schedule, the Fresh Gram Newton--Schulz routine induces a matrix polynomial map $A\mapsto \mathcal Q(A)$ on $\mathcal D$. Consequently, there exists a constant $L_Q>0$ such that
\[
\|\mathcal Q(A)-\mathcal Q(B)\|_{\mathrm{sp}}
\le
L_Q \|A-B\|_{\mathrm{sp}},
\qquad
A,B\in\mathcal D.
\]
Hence
\[
\|\hat D_t-\tilde D_t\|_{\mathrm{sp}}
\le
L_Q \tau_t,
\]
and therefore, with
\[
\delta_t := L_Q \tau_t + \varepsilon_{\mathrm{ref}},
\]
we have
\[
\|\hat D_t-D_t\|_{\mathrm{sp}} \le \delta_t
\qquad \text{for all } t.
\]
\end{proposition}

\begin{remark}
Proposition~\ref{prop:cache-inexact} isolates the cost of caching into two quantities: the fresh-solver error $\varepsilon_{\mathrm{ref}}$ and the Gram staleness $\tau_t$.
\end{remark}

\begin{theorem}[Stochastic convergence of CacheMuon]
\label{thm:main}
Under Assumptions \textbf{(A1)}--\textbf{(A4)}, suppose that
\[
\bar\delta := L_Q \bar\tau + \varepsilon_{\mathrm{ref}} < 1.
\]
Then for every integer $N\ge 1$,
\[
\begin{aligned}
\frac{1}{N}\sum_{t=1}^{N}\mathbb E\|\nabla F(W_t)\|_{\mathrm{tr}}
&\le
\frac{1}{1-\bar\delta}
\Bigg[
\frac{\Delta_1}{N\eta}
+ \frac{2\mathbb E[e_1]}{N(1-\beta)} 
+ \frac{2\beta L\eta(1+\bar\delta)}{1-\beta}
+ 2\rho \sigma \sqrt{1-\beta}\\
&\quad
+ \frac{L\eta}{2}(1+\bar\delta)^2
\Bigg]
\end{aligned}
\]
where
\[
\Delta_1 := F(W_1)-F_*,
\qquad
F_* := \inf_W F(W).
\]
Here \(\rho\) denotes a norm-equivalence constant satisfying
\(\|Z\|_{\mathrm{tr}} \le \rho \|Z\|_F\).
In particular, choosing
\[
\eta = \Theta(N^{-3/4}),
\qquad
1-\beta = \Theta(N^{-1/2})
\]
yields
\[
\frac{1}{N}\sum_{t=1}^{N}\mathbb E\|\nabla F(W_t)\|_{\mathrm{tr}}
=
O\!\left(\frac{1}{(1-\bar\delta)N^{1/4}}\right).
\]
\end{theorem}

\begin{remark}
Theorem~\ref{thm:main} shows that caching changes Muon only through the aggregate inexactness parameter $\bar\delta$. When every step refreshes, $\tau_t=0$ and the result reduces to the fresh-orthogonalization case up to the approximation error of the inner solver. More aggressive reuse affects the bound only through the induced drift term $L_Q\bar\tau$. All proofs are deferred to the appendix.
\end{remark}

%% file: sections/complexity_implementation.tex
\section{Complexity and Implementation}
\label{sec:complexity_implementation}

\paragraph{Practical Trigger.}

The convergence analysis in \Cref{sec:convergence} assumes bounded Gram staleness, suggesting that reuse is allowed only while the cached transform remains close to the current orthogonalization map. In practice, however, directly monitoring the theorem quantity \(\|A_t - A_{a_t}\|_{\mathrm{sp}}\) is costly. Instead, we use a cheap surrogate based on the orthogonality residual of the cached candidate \(\hat D_t^{\mathrm{cand}} = Q_{a_t} X_t\):
\[
\mathrm{err}_t
:=
\frac{\|\hat D_t^{\mathrm{cand}}(\hat D_t^{\mathrm{cand}})^\top - I\|_F}{\|I\|_F}.
\]
We accept reuse when \(\mathrm{err}_t \le \gamma\), where \(\gamma > 0\) is a prescribed threshold; otherwise, we refresh the cache with a FreshGNS solve. This criterion checks whether the cached transform still produces a nearly semi-orthogonal output and serves as a practical proxy for the theoretical staleness condition.

\paragraph{FLOPs Analysis.}

CacheMuon changes Muon only through the orthogonalization routine, so its computational effect can be described directly in terms of orthogonalization FLOPs. In this section, the baseline solver is the stabilized Gram Newton--Schulz (GNS) routine in \Cref{alg:gns}. Let \(F_{\mathrm{GNS}}\) denote the FLOPs of one call to \Cref{alg:gns}, and decompose
\[
F_{\mathrm{GNS}} = F_{\mathrm{norm}} + F_{\mathrm{rem}},
\]
where \(F_{\mathrm{norm}}\) is the matrix normalization cost shared by all calls and \(F_{\mathrm{rem}}\) is the remaining cost of the fresh solve. Relative to \Cref{alg:gns}, Cache Gram Newton--Schulz (CacheGNS) in \Cref{alg:cache-gns} introduces two extra costs: the probe cost \(\Delta_{\mathrm{probe}}\), consisting of applying the cached transform and evaluating the reuse criterion, and the cache-update cost \(\Delta_{\mathrm{cache}}\), consisting of accumulating and returning the refreshed transform \(Q_{\mathrm{tot}}\) after a miss. 

With this notation, a cache hit costs
\[
F_{\mathrm{hit}} = F_{\mathrm{norm}} + \Delta_{\mathrm{probe}},
\]
while a cache miss costs
\[
F_{\mathrm{miss}} = F_{\mathrm{norm}} + F_{\mathrm{rem}} + \Delta_{\mathrm{probe}} + \Delta_{\mathrm{cache}}.
\]

Ignoring the one-time cache-initialization call to Fresh Gram Newton--Schulz (FreshGNS), let \(p \in [0,1]\) denote the fraction of subsequent calls to CacheGNS that are cache hits. Then the average FLOPs per call are
\[
\bar F_{\mathrm{cache}}(p)
=
pF_{\mathrm{hit}} + (1-p)F_{\mathrm{miss}}
=
F_{\mathrm{norm}} + \Delta_{\mathrm{probe}} + (1-p)\bigl(F_{\mathrm{rem}} + \Delta_{\mathrm{cache}}\bigr).
\]
To determine when caching is beneficial, we compare this average cost against one fresh call to \Cref{alg:gns}, namely \(F_{\mathrm{GNS}} = F_{\mathrm{norm}} + F_{\mathrm{rem}}\), and require \(\bar F_{\mathrm{cache}}(p) < F_{\mathrm{GNS}}\). This gives
\[
p >
\frac{\Delta_{\mathrm{probe}} + \Delta_{\mathrm{cache}}}
     {F_{\mathrm{rem}} + \Delta_{\mathrm{cache}}}.
\]
Thus, caching reduces FLOPs when the hit rate is large enough to offset the probe cost and the cache-refresh overhead. This condition makes the tradeoff explicit: each hit avoids the post-normalization fresh-solver work \(F_{\mathrm{rem}}\), while each miss pays both the probe cost and the transform-accumulation overhead before falling back to the usual Gram Newton--Schulz work. Exact formulae for \(F_{\mathrm{norm}}, F_{\mathrm{rem}}, \Delta_{\mathrm{probe}}\), and \(\Delta_{\mathrm{cache}}\) are collected in Appendix~\ref{app:additional-experimental-details}, while \Cref{alg:cache-gns-flops} gives the corresponding line-by-line FLOPs annotations.

%% file: sections/numerical_experiments.tex
\section{Numerical Experiments}
\label{sec:numerical-experiments}

We evaluate CacheMuon from three complementary perspectives. First, we show the end-to-end quality vs efficiency tradeoff on GPT-2 Large by varying the residual threshold \(\gamma\). Second, we run diagnostic experiments to study the dynamics of cached updates over time. Third, we report supporting end-to-end results on GPT-2 Small and ResNet-18 / CIFAR-10 to show that the same qualitative behavior persists across both model scale and modality. Unless stated otherwise, GramMuon denotes Muon with a stabilized Gram Newton--Schulz (\Cref{alg:gns}) solve at every step, and PolarExpMuon denotes Muon with standard Newton--Schulz (\Cref{alg:standard-ns}) using PolarExpress coefficients \citep{muon,gram_newton_schulz,polar_express}. CacheMuon uses CacheGNS (\Cref{alg:cache-gns}) that runs FreshGNS on refresh steps, and reuses the cached transform on accepted steps.

\paragraph{GPT-2 Large quality--efficiency tradeoff.}
Our main experiment trains GPT-2 Large \citep{gpt2} on OpenWebText \citep{openwebtext} for 1B tokens and subsequently on 4.8B tokens. We first tune PolarExpMuon on a 100M-token proxy run and reuse the resulting hyperparameters for GramMuon and CacheMuon in the longer runs. For the 1B-token experiment, we vary only the residual threshold \(\gamma\in\{2,5,10,15\}\) in CacheMuon. This isolates the effect of temporal reuse: smaller \(\gamma\) forces more refreshes and stays closer to fresh orthogonalization, while larger \(\gamma\) permits more reuse and therefore larger orthogonalization-FLOPs savings. Because the forward/backward pass and all non-orthogonalization work are unchanged across GramMuon, PolarExpMuon, and CacheMuon, we report orthogonalization FLOPs savings as the efficiency metric. \Cref{fig:gpt2large-quality-flops} shows the resulting tradeoff. Conservative CacheMuon thresholds closely track the fresh baselines: \(\gamma=2\) reaches a best validation loss of 3.041, compared with 3.037 for GramMuon and 3.039 for PolarExpMuon, while still saving 13.6\% of orthogonalization FLOPs relative to GramMuon and 30.9\% relative to PolarExpMuon. Increasing \(\gamma\) gives a smooth efficiency frontier: \(\gamma=5,10,15\) save 42.5\%, 56.0\%, and 65.0\% relative to GramMuon, respectively, and 54.0\%, 64.8\%, and 72.0\% relative to PolarExpMuon. Even the most aggressive setting tested, \(\gamma=15\), reaches a best validation loss of 3.096. Thus, \(\gamma\) behaves as a single knob that interpolates between fresh Muon behavior and more aggressive orthogonalization reuse. A longer 4.8B-token continuation of the same conservative \(\gamma=2\) setting shows that CacheMuon continues to closely track the fresh baselines, essentially matching PolarExpMuon, while increasing orthogonalization-FLOPs savings from 13.6\% to 22\% relative to GramMuon and from 30.9\% to 38\% relative to PolarExpMuon (\Cref{fig:gpt2large-quality-4p8b-app}).

\begin{figure}[t]
    \centering
    \includegraphics[width=\linewidth]{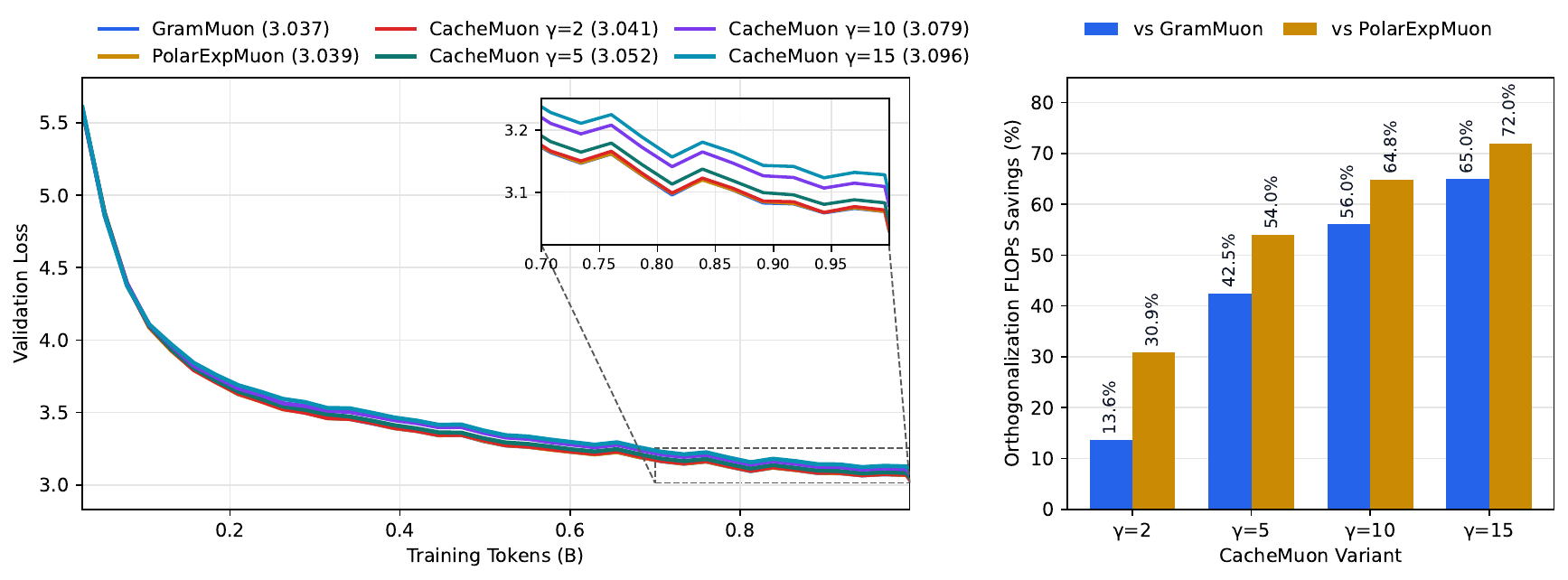}
    \caption{Left: validation loss over a 1B-token OpenWebText run for GramMuon, PolarExpMuon, and CacheMuon with different residual thresholds \(\gamma\). Numbers in parentheses denote the best validation loss reached by each run. Right: orthogonalization-FLOPs savings of each CacheMuon threshold relative to the GramMuon and PolarExpMuon. Lower thresholds closely track fresh orthogonalization, while larger thresholds trade a small validation-loss gap for larger orthogonalization savings.}
    \label{fig:gpt2large-quality-flops}
\end{figure}

\paragraph{Residual-gate diagnostics.}
We next test whether the practical residual trigger is a meaningful proxy for cached-update quality on a 100M-token GPT-2 Large run with \(\gamma=5\). \Cref{fig:gpt2large-diagnostics} helps explain the smooth frontier in \Cref{fig:gpt2large-quality-flops}. Early in training, cached candidates have large residuals and relatively weak agreement with the corresponding fresh update, so the trigger rejects most reuse attempts. As optimization stabilizes, the cached candidate becomes both more orthogonal and more aligned with the fresh update, which creates more valid reuse opportunities. \Cref{fig:gpt2large-diagnostics}(b) is particularly informative: the accepted-only curve lies consistently above the all-probe median, showing that the trigger selects a subset of cached updates that is systematically closer to fresh orthogonalization than the full population of probes. \Cref{fig:gpt2large-diagnostics}(c) further shows that lower residual is associated with higher cached/fresh cosine, so \(\gamma\) acts as a practical fidelity knob rather than an arbitrary threshold. Finally, \Cref{fig:gpt2large-diagnostics}(d) shows the computational consequence: once the momentum changes more smoothly, accepted reuse becomes more frequent and orthogonalization work is reduced, suggesting that CacheMuon can be even more efficient for longer runs where cache hits are frequent.

\begin{figure}[t]
    \centering
    \begin{minipage}[t]{0.49\linewidth}
        \centering
        \includegraphics[width=\linewidth]{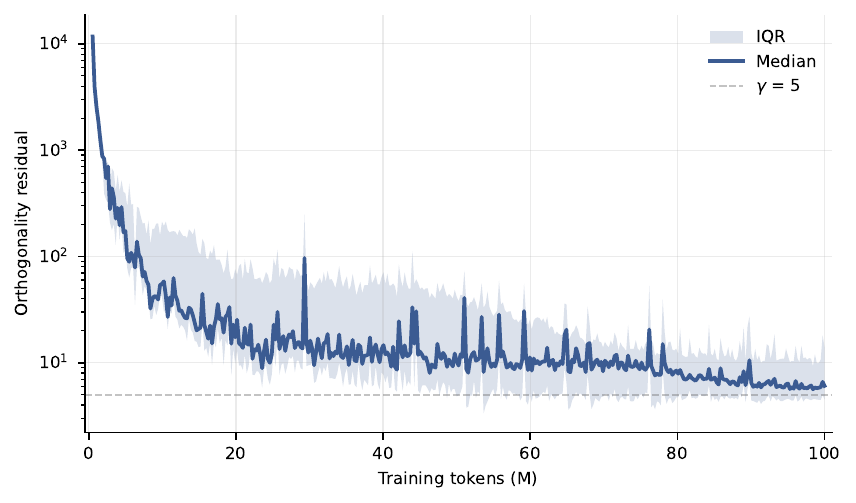}
        {\footnotesize \textbf{(a)} Cached-candidate residual over training.}
    \end{minipage}\hfill
    \begin{minipage}[t]{0.49\linewidth}
        \centering
        \includegraphics[width=\linewidth]{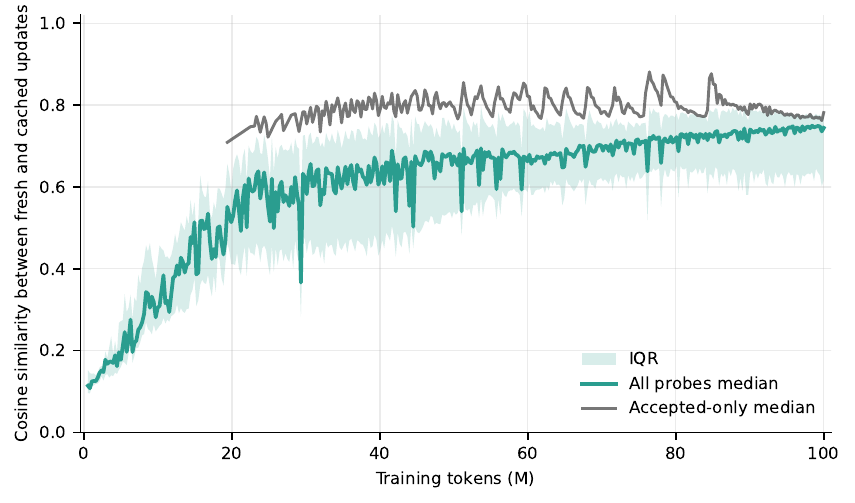}
        {\footnotesize \textbf{(b)} Cached/fresh cosine over training.}
    \end{minipage}

    \vspace{0.6\baselineskip}

    \begin{minipage}[t]{0.49\linewidth}
        \centering
        \includegraphics[width=\linewidth]{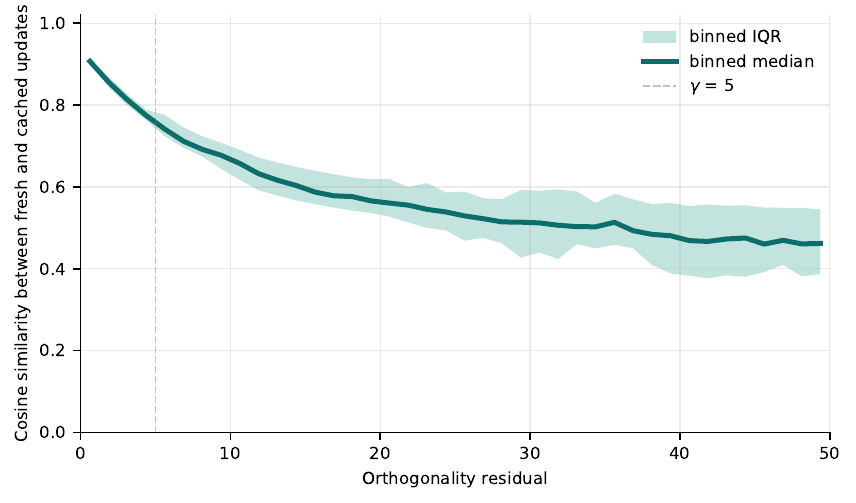}
        {\footnotesize \textbf{(c)} Residual versus cached/fresh cosine.}
    \end{minipage}\hfill
    \begin{minipage}[t]{0.49\linewidth}
        \centering
        \includegraphics[width=\linewidth]{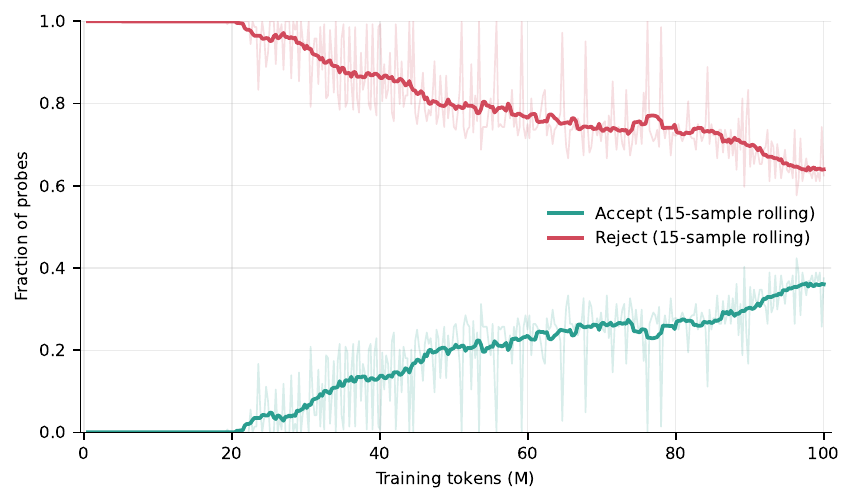}
        {\footnotesize \textbf{(d)} Smoothed cache accept/reject rates over training.}
    \end{minipage}
    \caption{Cache diagnostics from a 100M-token GPT-2 Large / OpenWebText run with CacheMuon (\(\gamma=5\)). Residuals decrease over training (a); cached and fresh updates become more aligned over time, and the accepted subset is consistently better aligned than the full probe set (b); lower residual is associated with higher cached/fresh cosine (c); and reuse is accepted more often later in training (d). Together, these trends show that the residual gate is a useful proxy for cached-update fidelity.}
    \label{fig:gpt2large-diagnostics}
\end{figure}

\paragraph{Additional training runs.}
We also evaluate CacheMuon on GPT-2 Small / OpenWebText and ResNet-18 / CIFAR-10 \citep{resnet_deepresidual,cifar10} as supporting evidence across scale and modality. For GPT-2 Small, we follow the same 100M tune and 1B run protocol used above; the ResNet-18 runs are tuned separately. In both settings, CacheMuon remains close to the fresh baselines while reducing cumulative orthogonalization work. On GPT-2 Small, CacheMuon saves 28\% of orthogonalization FLOPs relative to GramMuon and 43\% relative to PolarExpMuon; on ResNet-18, it saves 15\% and 51\%, respectively. Full training curves, cumulative-FLOPs tables, and further diagnostic runs are deferred to \Cref{app:additional-experimental-details}.

\FloatBarrier

%% file: sections/discussion_and_limitation.tex
\section{Discussion and Limitations}
One important caveat of our approach is that reductions in theoretical FLOPs do not directly translate to proportional wall-clock speedups in our current high-level implementation. In our GPT-2 Large runs, CacheMuon yields only modest end-to-end runtime improvements even when the orthogonalization-FLOPs savings are substantial. Cache reuse requires probing the validity of the cached transform at each step, which introduces host--device synchronization and creates a serialization point in the execution pipeline. As a result, although our method lowers orthogonalization FLOPs, the realized speedup remains constrained by systems-level overheads and the rest of the training stack rather than arithmetic cost alone. A more optimized low-level implementation---for example, by fusing the probe and reuse criterion into device-side kernels or avoiding explicit synchronization---could translate a larger fraction of the arithmetic savings into wall-clock gains.

%% file: sections/conclusion.tex
\section{Conclusion}
We introduced CacheMuon, a Muon variant that amortizes polar-factor computation by reusing cached transforms across optimization steps. CacheMuon changes only the orthogonalization routine, using a residual gate to decide when reuse remains accurate and when a fresh solve is needed. We analyzed the method as an inexact Muon update, with error controlled by fresh-solver accuracy and cache staleness, and established a stochastic convergence guarantee under standard assumptions. Empirically, CacheMuon provides a practical quality--efficiency tradeoff: conservative reuse closely matches fresh Muon while reducing orthogonalization FLOPs, and more aggressive reuse yields larger savings with modest quality degradation. These results suggest that temporal reuse is a promising direction for reducing the cost of orthogonalization in Muon.

%% file: sections/appendix1.tex
\section{Algorithms}
\label{app:algorithms}

For convenience, we collect all pseudocode in one place. The standard Newton--Schulz routine used in Muon (\cite{muon}) is given in \Cref{alg:standard-ns}. PolarExpress (\cite{polar_express}) also uses the same update structure, but with optimized coefficients $\{(a_k,b_k,c_k)\}_{k=1}^K$.

\renewcommand{\thealgorithm}{A.\arabic{algorithm}}
\setcounter{algorithm}{0}

\begin{algorithm}[h]
\caption{StandardNS: Standard Newton--Schulz}\label{alg:standard-ns}
\begin{algorithmic}[1]
\Require matrix $M \in \mathbb{R}^{n\times m}$, coefficients $\{(a_k,b_k,c_k)\}_{k=1}^K$, $\varepsilon > 0$
\State $X \gets M / (\|M\|_F + \varepsilon)$
\If{$n > m$}
    \State $X \gets X^\top$
\EndIf
\For{$k=1,\ldots,K$}
    \State $A \gets X X^\top$
    \State $P \gets a_k I + b_k A + c_k A^2$
    \State $X \gets P X$
\EndFor
\If{$n > m$}
    \State $X \gets X^\top$
\EndIf
\State \Return $X$
\end{algorithmic}
\end{algorithm}

The stabilized Gram Newton--Schulz routine used in GramMuon (\cite{gram_newton_schulz}) is summarized in \Cref{alg:gns}. When the restart set satisfies $\mathcal{S} = \emptyset$, \Cref{alg:gns} reduces to the original Gram Newton--Schulz iteration, which is unstable and suffers from numerical issues in finite precision.

\begin{algorithm}[h]
\caption{GramNS: Stabilized Gram Newton--Schulz}
\label{alg:gns}
\begin{algorithmic}[1]
\Require matrix \(M \in \mathbb{R}^{n\times m}\), coefficients \(\{(a_k,b_k,c_k)\}_{k=1}^K\), restart set \(\mathcal S \subseteq \{2,\dots,K\}\), \(\varepsilon > 0\)
\State \(X \gets M / (\|M\|_F + \varepsilon)\)
\If{\(n > m\)}
    \State \(X \gets X^\top\)
\EndIf
\State \(R \gets X X^\top\)
\State \(Q_{\mathrm{loc}} \gets I\)
\For{\(k = 1,\dots,K\)}
    \If{\(k \in \mathcal S\)}
        \State \(X \gets Q_{\mathrm{loc}} X\)
        \State \(R \gets X X^\top\)
        \State \(Q_{\mathrm{loc}} \gets I\)
    \EndIf
    \State \(Z \gets b_k R + c_k R^2\)
    \State \(Q_{\mathrm{loc}} \gets Q_{\mathrm{loc}} Z + a_k Q_{\mathrm{loc}}\)
    \If{\(k < K\) and \(k+1 \notin \mathcal S\)}
        \State \(RZ \gets R Z + a_k R\)
        \State \(R \gets Z RZ + a_k RZ\)
    \EndIf
\EndFor
\State \(X \gets Q_{\mathrm{loc}} X\)
\If{\(n > m\)}
    \State \(X \gets X^\top\)
\EndIf
\State \Return \(X\)
\end{algorithmic}
\end{algorithm}

We next modify stabilized Gram Newton--Schulz to keep track of the cumulative transform \(Q_{\mathrm{tot}}\), which maps the normalized input \(X\) to its approximate polar factor. The resulting procedure is shown in \Cref{alg:fresh-gns}, where the colored lines highlight the modifications. Relative to \Cref{alg:gns}, FreshGNS in \Cref{alg:fresh-gns} incurs the additional cost of accumulating \(Q_{\mathrm{tot}}\). In CacheMuon, we use FreshGNS at the first step to initialize the cached transform and then use CacheGNS in \Cref{alg:cache-gns} at subsequent steps to reuse that cache. The full CacheMuon procedure is given in \Cref{alg:cache-muon}.

\begin{algorithm}[h]
\caption{FreshGNS: Fresh stabilized Gram Newton--Schulz}
\label{alg:fresh-gns}
\begin{algorithmic}[1]
\Require matrix \(M \in \mathbb{R}^{n\times m}\), coefficients \(\{(a_k,b_k,c_k)\}_{k=1}^K\), restart set \(\mathcal S \subseteq \{2,\dots,K\}\), \(\varepsilon > 0\)
\State \(X \gets M / (\|M\|_F + \varepsilon)\)
\If{\(n > m\)}
    \State \(X \gets X^\top\)
\EndIf
\State \textcolor{blue}{\(\mathsf{hasTot} \gets \mathrm{false}\)}
\State \(R \gets XX^\top\)
\State \(Q_{\mathrm{loc}} \gets I\)
\For{\(k = 1,\dots,K\)}
    \If{\(k \in \mathcal S\)}
        \State \(X \gets Q_{\mathrm{loc}} X\)
        \If{\textcolor{blue}{\(\mathsf{hasTot\ is\ false}\)}}
            \State \textcolor{blue}{\(Q_{\mathrm{tot}} \gets Q_{\mathrm{loc}}\)}
            \State \textcolor{blue}{\(\mathsf{hasTot} \gets \mathrm{true}\)}
        \Else
            \State \textcolor{blue}{\(Q_{\mathrm{tot}} \gets Q_{\mathrm{loc}} Q_{\mathrm{tot}}\)}
        \EndIf
        \State \(R \gets XX^\top\)
        \State \(Q_{\mathrm{loc}} \gets I\)
    \EndIf
    \State \(Z \gets b_k R + c_k R^2\)
    \State \(Q_{\mathrm{loc}} \gets Q_{\mathrm{loc}} Z + a_k Q_{\mathrm{loc}}\)
    \If{\(k < K\) and \(k+1 \notin \mathcal S\)}
        \State \(RZ \gets R Z + a_k R\)
        \State \(R \gets Z RZ + a_k RZ\)
    \EndIf
\EndFor
\State \(X \gets Q_{\mathrm{loc}} X\)
\If{\textcolor{blue}{\(\mathsf{hasTot\ is\ false}\)}}
    \State \textcolor{blue}{\(Q_{\mathrm{tot}} \gets Q_{\mathrm{loc}}\)}
\Else
    \State \textcolor{blue}{\(Q_{\mathrm{tot}} \gets Q_{\mathrm{loc}} Q_{\mathrm{tot}}\)}
\EndIf
\If{\(n > m\)}
    \State \(X \gets X^\top\)
\EndIf
\State \Return \(X, \textcolor{blue}{Q_{\mathrm{tot}}}\)
\end{algorithmic}
\end{algorithm}

\begin{algorithm}[h]
\caption{CacheGNS: Cache Gram Newton--Schulz}
\label{alg:cache-gns}
\begin{algorithmic}[1]
\Require matrix \(M \in \mathbb{R}^{n\times m}\) with \(n \le m\), \textcolor{blue}{cached transform \(Q_{\mathrm{cache}} \in \mathbb{R}^{n\times n}\)}, coefficients \(\{(a_k,b_k,c_k)\}_{k=1}^K\), restart set \(\mathcal S \subseteq \{2,\dots,K\}\), \textcolor{blue}{threshold \(\gamma > 0\)}, \(\varepsilon > 0\)
\State \(X \gets M / (\|M\|_F + \varepsilon)\)
\State \textcolor{blue}{\(X_{\mathrm{cand}} \gets Q_{\mathrm{cache}} X\)}
\If{\textcolor{blue}{\(\|X_{\mathrm{cand}}X_{\mathrm{cand}}^\top - I\|_F / \|I\|_F \le \gamma\)}}
    \State \Return \textcolor{blue}{\(X_{\mathrm{cand}}, Q_{\mathrm{cache}}\)}
\EndIf
\State \textcolor{blue}{\(\mathsf{hasTot} \gets \mathrm{false}\)}
\State \(R \gets X X^\top\)
\State \(Q_{\mathrm{loc}} \gets I\)
\For{\(k = 1,\dots,K\)}
    \If{\(k \in \mathcal S\)}
        \State \(X \gets Q_{\mathrm{loc}} X\)
        \If{\textcolor{blue}{\(\mathsf{hasTot\ is\ false}\)}}
            \State \textcolor{blue}{\(Q_{\mathrm{tot}} \gets Q_{\mathrm{loc}}\)}
            \State \textcolor{blue}{\(\mathsf{hasTot} \gets \mathrm{true}\)}
        \Else
            \State \textcolor{blue}{\(Q_{\mathrm{tot}} \gets Q_{\mathrm{loc}} Q_{\mathrm{tot}}\)}
        \EndIf
        \State \(R \gets X X^\top\)
        \State \(Q_{\mathrm{loc}} \gets I\)
    \EndIf
    \State \(Z \gets b_k R + c_k R^2\)
    \State \(Q_{\mathrm{loc}} \gets Q_{\mathrm{loc}} Z + a_k Q_{\mathrm{loc}}\)
    \If{\(k < K\) and \(k+1 \notin \mathcal S\)}
        \State \(RZ \gets R Z + a_k R\)
        \State \(R \gets Z RZ + a_k RZ\)
    \EndIf
\EndFor
\State \(X \gets Q_{\mathrm{loc}} X\)
\If{\textcolor{blue}{\(\mathsf{hasTot\ is\ false}\)}}
    \State \textcolor{blue}{\(Q_{\mathrm{tot}} \gets Q_{\mathrm{loc}}\)}
\Else
    \State \textcolor{blue}{\(Q_{\mathrm{tot}} \gets Q_{\mathrm{loc}} Q_{\mathrm{tot}}\)}
\EndIf
\State \Return \(X, Q_{\mathrm{tot}}\)
\end{algorithmic}
\end{algorithm}

\begin{algorithm}[h]
\caption{CacheMuon}\label{alg:cache-muon}
\begin{algorithmic}[1]
\Require stepsize $\eta > 0$, momentum $\beta \in [0,1)$, coefficients $\{(a_k,b_k,c_k)\}_{k=1}^K$, restart set $\mathcal S \subseteq \{2,\dots,K\}$, threshold $\gamma > 0$, $\varepsilon > 0$
\State Initialize $M_0 \gets 0$

\State $g_1 \gets \nabla f(W_1;\xi_1)$
\State $M_1 \gets (1-\beta) g_1 + \beta M_0$
\State $\hat D_1, Q_1 \gets \FreshGNS(M_1, \{(a_k,b_k,c_k)\}_{k=1}^K, \mathcal S, \varepsilon)$ \Comment{Run fresh GNS to get an initial cache}
\State $W_2 \gets W_1 - \eta \hat D_1$

\For{$t=2,\ldots$}
    \State $g_t \gets \nabla f(W_t;\xi_t)$
    \State $M_t \gets (1-\beta) g_t + \beta M_{t-1}$
    \State $\hat D_t, Q_t \gets \CacheGNS(M_t, Q_{t-1}, \{(a_k,b_k,c_k)\}_{k=1}^K, \mathcal S, \gamma, \varepsilon)$ \Comment{Use cache}
    \State $W_{t+1} \gets W_t - \eta \hat D_t$
\EndFor
\end{algorithmic}
\end{algorithm}

For completeness, \Cref{alg:cache-gns-flops} presents CacheGNS again with explicit FLOPs annotations for the operations used in the complexity discussion.

\begin{algorithm}[h]
\caption{Cache Gram Newton--Schulz with FLOPs annotations}
\label{alg:cache-gns-flops}
\begin{algorithmic}[1]
\Require matrix \(M \in \mathbb{R}^{n\times m}\) with \(n \le m\), \textcolor{blue}{cached transform \(Q_{\mathrm{cache}} \in \mathbb{R}^{n\times n}\)}, coefficients \(\{(a_k,b_k,c_k)\}_{k=1}^K\), restart set \(\mathcal S \subseteq \{2,\dots,K\}\), \textcolor{blue}{threshold \(\gamma > 0\)}, \(\varepsilon > 0\)
\State \(X \gets M / (\|M\|_F + \varepsilon)\) \Comment{\((3nm + 1)\ \mathrm{FLOPs}\)}
\State \textcolor{blue}{\(X_{\mathrm{cand}} \gets Q_{\mathrm{cache}} X\)} \Comment{\textcolor{blue}{\(nm(2n-1)\ \mathrm{FLOPs}\)}}
\If{\textcolor{blue}{\(\|X_{\mathrm{cand}}X_{\mathrm{cand}}^\top - I\|_F / \|I\|_F \le \gamma\)}} \Comment{\(\textcolor{blue}{(2n^2m + n^2 + n + 1)\ \mathrm{FLOPs}}\)}
    \State \Return \textcolor{blue}{\(X_{\mathrm{cand}}, Q_{\mathrm{cache}}\)}
\EndIf
\State \textcolor{blue}{\(\mathsf{hasTot} \gets \mathrm{false}\)}
\State \(R \gets X X^\top\) \Comment{\(n^2(2m-1)\ \mathrm{FLOPs}\)}
\State \(Q_{\mathrm{loc}} \gets I\)
\For{\(k = 1,\dots,K\)}
    \If{\(k \in \mathcal S\)}
        \State \(X \gets Q_{\mathrm{loc}} X\) \Comment{\(nm(2n-1)\ \mathrm{FLOPs}\)}
        \If{\textcolor{blue}{\(\mathsf{hasTot\ is\ false}\)}}
            \State \textcolor{blue}{\(Q_{\mathrm{tot}} \gets Q_{\mathrm{loc}}\)}
            \State \textcolor{blue}{\(\mathsf{hasTot} \gets \mathrm{true}\)}
        \Else
            \State \textcolor{blue}{\(Q_{\mathrm{tot}} \gets Q_{\mathrm{loc}} Q_{\mathrm{tot}}\)} \Comment{\textcolor{blue}{\(n^2(2n-1)\ \mathrm{FLOPs}\)}}
        \EndIf
        \State \(R \gets X X^\top\) \Comment{\(n^2(2m-1)\ \mathrm{FLOPs}\)}
        \State \(Q_{\mathrm{loc}} \gets I\)
    \EndIf
    \State \(Z \gets b_k R + c_k R^2\) \Comment{\(n^2(2n+2)\ \mathrm{FLOPs}\)}
    \State \(Q_{\mathrm{loc}} \gets Q_{\mathrm{loc}} Z + a_k Q_{\mathrm{loc}}\) \Comment{\(n^2(2n+1)\ \mathrm{FLOPs}\)}
    \If{\(k < K\) and \(k+1 \notin \mathcal S\)}
        \State \(RZ \gets R Z + a_k R\) \Comment{\(n^2(2n+1)\ \mathrm{FLOPs}\)}
        \State \(R \gets Z RZ + a_k RZ\) \Comment{\(n^2(2n+1)\ \mathrm{FLOPs}\)}
    \EndIf
\EndFor
\State \(X \gets Q_{\mathrm{loc}} X\) \Comment{\(nm(2n-1)\ \mathrm{FLOPs}\)}
\If{\textcolor{blue}{\(\mathsf{hasTot\ is\ false}\)}}
    \State \textcolor{blue}{\(Q_{\mathrm{tot}} \gets Q_{\mathrm{loc}}\)}
\Else
    \State \textcolor{blue}{\(Q_{\mathrm{tot}} \gets Q_{\mathrm{loc}} Q_{\mathrm{tot}}\)} \Comment{\textcolor{blue}{\(n^2(2n-1)\ \mathrm{FLOPs}\)}}
\EndIf
\State \Return \(X, Q_{\mathrm{tot}}\)
\end{algorithmic}
\end{algorithm}

%% file: sections/appendix2.tex
\section{Additional Experimental Results and Details}
\label{app:additional-experimental-details}

This appendix collects the supporting experiments and implementation details for Section~\ref{sec:numerical-experiments}. The main text focuses on the GPT-2 Large quality--efficiency frontier and the residual-gate diagnostics. Here we provide the ResNet-18 staleness diagnostic, the GPT-2 Small and ResNet-18 end-to-end training curves together with the supporting FLOPs table, and the full training and configuration details.

\subsection{Additional Diagnostic and Training-Quality Results}
\label{app:additional-experimental-results}

\paragraph{Staleness versus update agreement.}
We measure whether cached transforms remain useful across nearby optimizer steps in a ResNet-18 / CIFAR-10 diagnostic run. To that end, we train ResNet-18 \citep{resnet_deepresidual} on CIFAR-10 \citep{cifar10} using fresh Gram Newton--Schulz Muon at every step, and use the cache only for logging. For each matrix parameter and step \(t\), we compute the fresh direction \(\tilde D_t\) and fresh left transform \(Q_t\). For an anchor age \(h\), we form the hypothetical cached direction \(\hat D_t^{(h)} = Q_{t-h}X_t\), where \(X_t\) is the current normalized momentum, and compare \(\hat D_t^{(h)}\) with \(\tilde D_t\). Figure~\ref{fig:age-cosine} reports the cosine similarity \(\cos(\mathrm{vec}(\hat D_t^{(h)}),\mathrm{vec}(\tilde D_t))\) across anchor ages. The steady decay with anchor age shows that stale transforms remain aligned for nearby steps but gradually lose agreement.

\begin{figure}[htbp]
    \centering
    \includegraphics[width=0.65\linewidth]{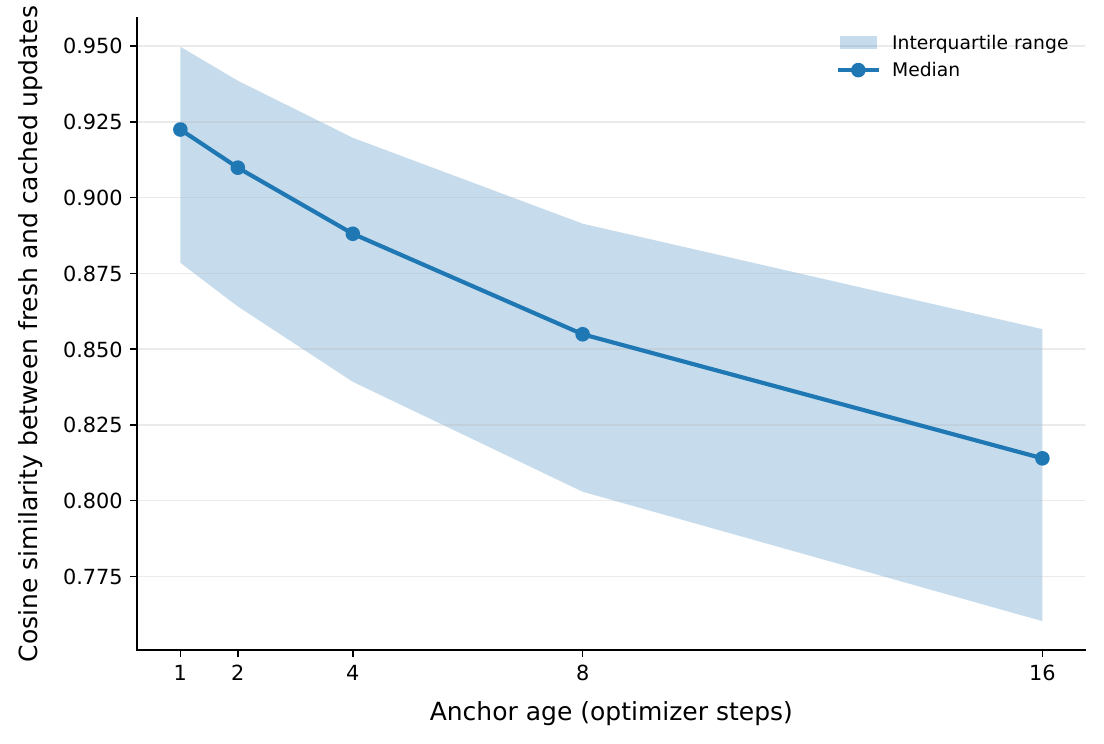}
    \caption{Cosine similarity between the hypothetical cached direction \(\hat D_t^{(h)} = Q_{t-h}X_t\) and the fresh direction \(\tilde D_t\) as a function of anchor age \(h\). Agreement remains high for recent anchor ages and decays as the cached transform becomes more stale.}
    \label{fig:age-cosine}
\end{figure}

\Cref{fig:age-cosine-layerwise} provides a layerwise breakdown of the staleness-versus-agreement diagnostic from \Cref{fig:age-cosine} for the ResNet-18 / CIFAR-10 experiment.

\begin{figure}[htbp]
\centering
\includegraphics[width=0.92\linewidth,height=0.92\textheight,keepaspectratio]{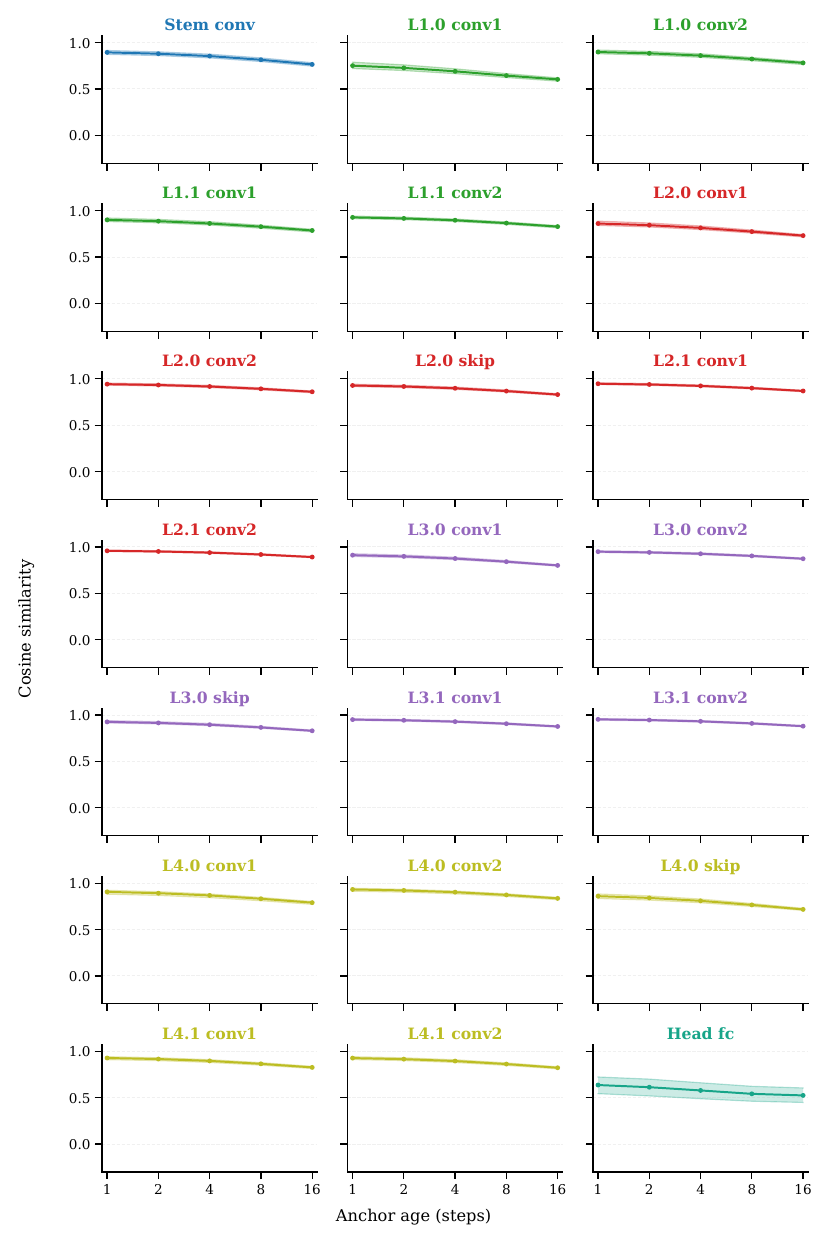}
\vspace{4pt}
\caption{Layerwise cosine similarity between the hypothetical cached direction \(\hat D_t^{(h)} = Q_{t-h}X_t\) and the fresh direction \(\tilde D_t\) for the ResNet-18 / CIFAR-10 experiment underlying \Cref{fig:age-cosine}. Each panel reports the same anchor-age diagnostic as \Cref{fig:age-cosine} for one parameter group.}
\label{fig:age-cosine-layerwise}
\end{figure}

\paragraph{GPT-2 Small and ResNet-18 training quality.}
\Cref{fig:additional-end-to-end} reports end-to-end training quality for CacheMuon, GramMuon, and PolarExpMuon on GPT-2 Small / OpenWebText and ResNet-18 / CIFAR-10. For GPT-2 Small \citep{gpt2} on OpenWebText \citep{openwebtext}, we first tune the hyperparameters on a 100M-token proxy run using PolarExpMuon, and then reuse that same configuration for GramMuon and CacheMuon when running the full 1B-token experiment. Separate tuning sweeps for GramMuon and CacheMuon recover the same optimal hyperparameters, so the reported GPT-2 Small runs ultimately differ only in the orthogonalization rule and the CacheMuon threshold \(\gamma\). The ResNet-18 vision runs are tuned separately. In both settings, CacheMuon closely tracks the fresh-orthogonalization baselines.

\begin{figure}[htbp]
    \centering
    \begin{minipage}[t]{0.49\linewidth}
        \includegraphics[width=\linewidth,height=0.21\textheight,keepaspectratio]{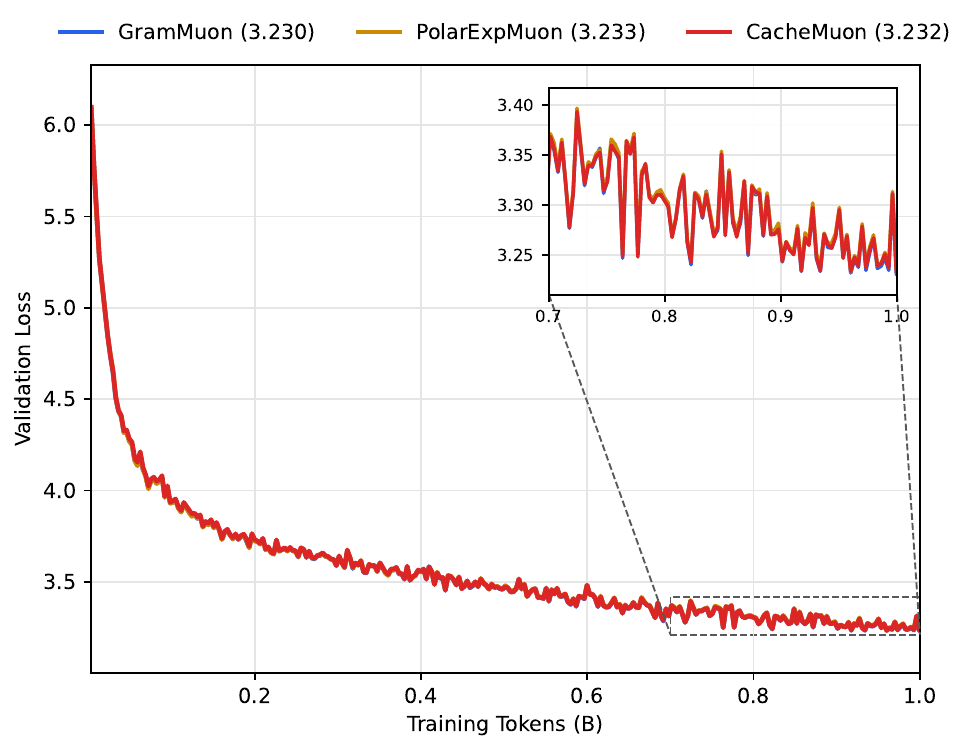}
    \end{minipage}
    \begin{minipage}[t]{0.49\linewidth}
        \includegraphics[width=\linewidth,height=0.21\textheight,keepaspectratio]{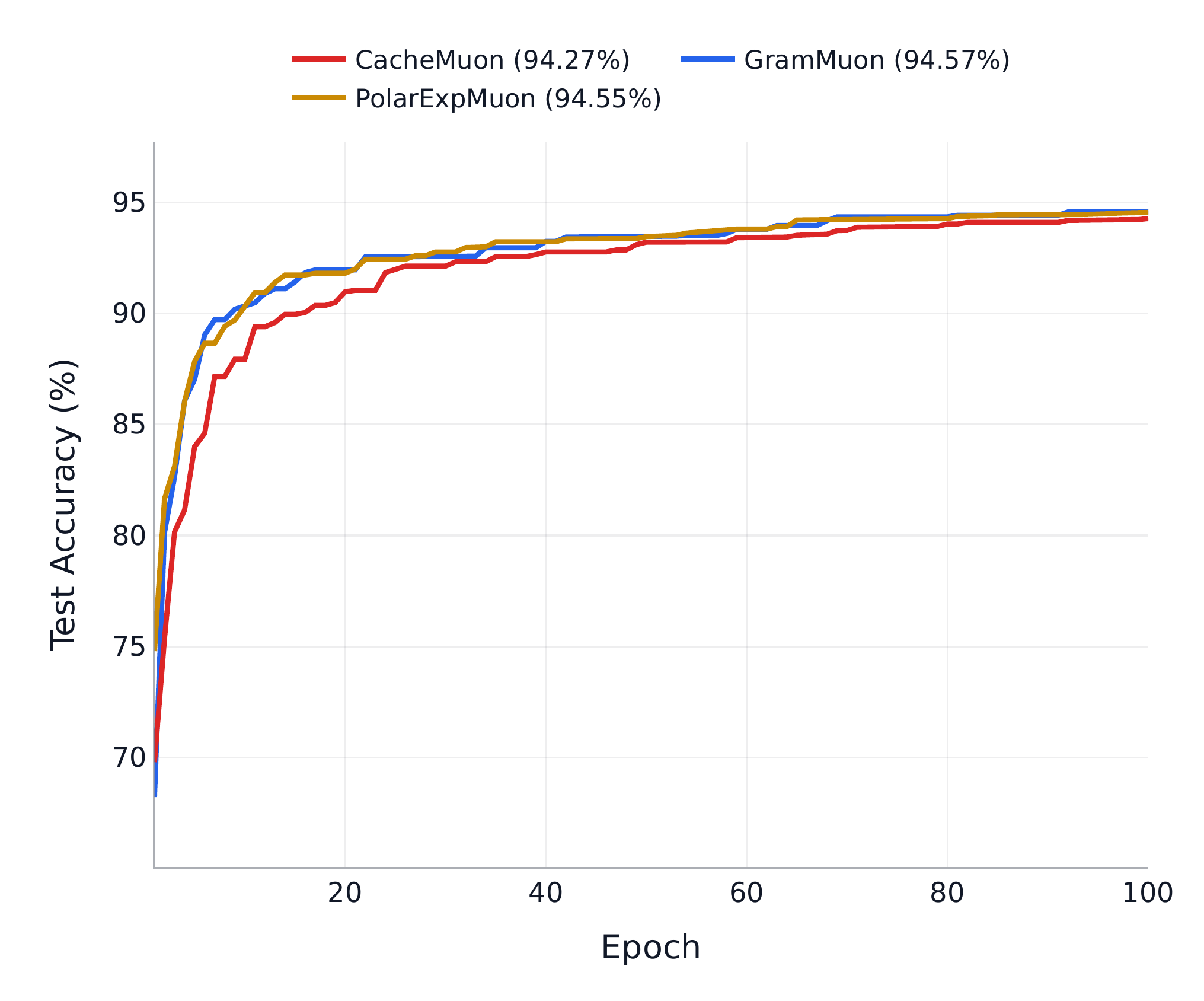}
    \end{minipage}
\caption{End-to-end training quality of CacheMuon, GramMuon, and PolarExpMuon on GPT-2 Small / OpenWebText (left) and ResNet-18 / CIFAR-10 (right). CacheMuon remains competitive with the GramMuon and PolarExpMuon baselines in both settings.}
    \label{fig:additional-end-to-end}
\end{figure}

\begin{figure}[htbp]
    \centering
    \includegraphics[width=0.72\linewidth]{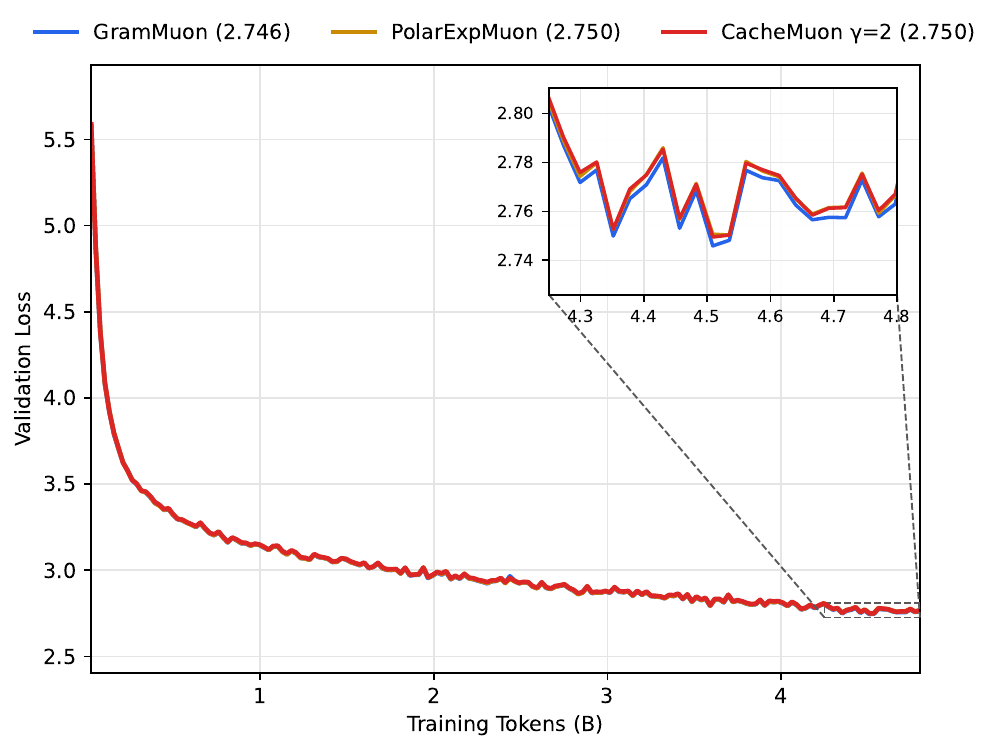}
    \caption{Longer 4.8B-token continuation of the GPT-2 Large / OpenWebText \(\gamma=2\) run from \Cref{fig:gpt2large-quality-flops}, using the same hyperparameters as the 1B experiment. CacheMuon continues to track GramMuon and PolarExpMuon closely while increasing orthogonalization-FLOPs savings to 22\% relative to GramMuon and 38\% relative to PolarExpMuon.}
    \label{fig:gpt2large-quality-4p8b-app}
\end{figure}

\paragraph{Orthogonalization cost formulae.}
To connect the decomposition in \Cref{sec:complexity_implementation} with the line-by-line annotations in \Cref{alg:cache-gns-flops}, we record the corresponding symbolic costs for the \(n \le m\) case used by CacheGNS. Write \(s := |\mathcal S|\) for the number of Gram Newton--Schulz restarts. Then the shared normalization cost is
\[
F_{\mathrm{norm}} = 3nm + 1.
\]

The additional probe overhead incurred before deciding whether to reuse the cache is
\[
\Delta_{\mathrm{probe}}
=
nm(2n-1) + (2n^2m + n^2 + n + 1),
\]
where the first term is the cached-transform application and the second term is the residual check.

The cache-refresh overhead incurred on a miss is
\[
\Delta_{\mathrm{cache}} = s\,n^2(2n-1),
\]
which comes from accumulating \(Q_{\mathrm{tot}}\) across the \(s\) restart boundaries.

For the fresh solver, the post-normalization cost is
\begin{align*}
F_{\mathrm{rem}}
&= (s+1)n^2(2m-1) \\
&\quad + (s+1)nm(2n-1) \\
&\quad + K n^2(2n+2) \\
&\quad + (3K-2-2s)\,n^2(2n+1).
\end{align*}
Here the first two lines account for the \(s+1\) Gram rebuilds and segment-ending applications of the local transform to \(X\), while the last two lines collect the \(K\) polynomial updates and the \(3K-2-2s\) recurrence updates that remain after excluding the skipped restart transitions.

Therefore,
\[
F_{\mathrm{GNS}} = F_{\mathrm{norm}} + F_{\mathrm{rem}},
\]
and the per-call CacheGNS costs are
\[
F_{\mathrm{hit}} = F_{\mathrm{norm}} + \Delta_{\mathrm{probe}},
\]
\[
F_{\mathrm{miss}} = F_{\mathrm{norm}} + F_{\mathrm{rem}} + \Delta_{\mathrm{probe}} + \Delta_{\mathrm{cache}}.
\]
These are the symbolic per-call quantities whose realized cumulative totals are measured in \Cref{tab:additional-run-flops}.

\paragraph{Full-run orthogonalization FLOPs comparison.}
\Cref{tab:additional-run-flops} reports the realized cumulative orthogonalization FLOPs for the GPT-2 Large in \Cref{fig:gpt2large-quality-flops} together with the GPT-2 Small and ResNet18 supporting runs used in \Cref{fig:additional-end-to-end}. Since the forward/backward pass and all non-orthogonalization optimizer work are identical across GramMuon, PolarExpMuon, and CacheMuon, the table isolates only the orthogonalization cost in PFLOPs. It lists the cumulative orthogonalization cost for GramMuon, PolarExpMuon, and CacheMuon on GPT-2 Large and GPT-2 Small over 1B OpenWebText tokens and on ResNet-18 over 100 CIFAR-10 epochs, together with the relative reduction of CacheMuon against each fresh-orthogonalization baseline and the overall cache hit rate attained by CacheMuon over the run.

\begin{table}[htbp]
    \centering
    \footnotesize
    \renewcommand{\arraystretch}{1.0}
    \setlength{\tabcolsep}{2pt}
    \caption{Measured cumulative orthogonalization cost for the GPT-2 Large in Figure~\ref{fig:gpt2large-quality-flops} and the GPT-2 Small and ResNet18 runs used in Figure~\ref{fig:additional-end-to-end}. FLOPs are reported in PFLOPs; CacheMuon savings are reported as percentages relative to the corresponding fresh-orthogonalization baseline.}
    \label{tab:additional-run-flops}
    \begin{tabular}{@{}>{\raggedright\arraybackslash}p{0.27\linewidth}@{\hspace{2pt}}>{\centering\arraybackslash}p{0.105\linewidth}>{\centering\arraybackslash}p{0.12\linewidth}>{\centering\arraybackslash}p{0.105\linewidth}@{\hspace{4pt}}>{\centering\arraybackslash}p{0.11\linewidth}>{\centering\arraybackslash}p{0.11\linewidth}@{\hspace{2pt}}>{\centering\arraybackslash}p{0.06\linewidth}@{}}
        \toprule
        & \multicolumn{3}{c}{Cumulative orthogonalization cost (PFLOPs)} & \multicolumn{2}{c}{CacheMuon savings} & \\
        \cmidrule(lr){2-4} \cmidrule(lr){5-6}
        \multicolumn{1}{@{}l}{Task and run} & \multicolumn{1}{c}{GramMuon} & \multicolumn{1}{c}{\shortstack[c]{PolarExp-\\Muon}} & \multicolumn{1}{c}{CacheMuon} & \multicolumn{1}{c}{\shortstack[c]{vs.\ Gram-\\Muon}} & \multicolumn{1}{c}{\shortstack[c]{vs.\ PolarExp-\\Muon}} & \multicolumn{1}{c@{}}{\shortstack[c]{Hit\\rate}} \\
        \midrule
        \shortstack[l]{GPT-2 Large / OpenWebText \\ 1B tokens, \(\gamma=2\)} & 516.16 & 644.99 & 445.97 & 13.6\% & 30.9\% & 33\% \\
        \shortstack[l]{GPT-2 Large / OpenWebText \\ 1B tokens, \(\gamma=5\)} & 516.16 & 644.99 & 296.90 & 42.5\% & 54.0\% & 59\% \\
        \shortstack[l]{GPT-2 Large / OpenWebText \\ 1B tokens, \(\gamma=10\)} & 516.16 & 644.99 & 226.88 & 56.0\% & 64.8\% & 73\% \\
        \shortstack[l]{GPT-2 Large / OpenWebText \\ 1B tokens, \(\gamma=15\)} & 516.16 & 644.99 & 180.76 & 65.0\% & 72.0\% & 84\% \\
        \addlinespace[2pt]
        \shortstack[l]{GPT-2 Small / OpenWebText \\ 1B tokens} & 37.18 & 46.45 & 26.67 & 28\% & 43\% & 45\% \\
        \shortstack[l]{ResNet-18 / CIFAR-10 \\ 100 epochs} & 2.32 & 4.06 & 1.98 & 15\% & 51\% & 80\% \\
        \bottomrule
    \end{tabular}
\end{table}

For GPT-2 Large, increasing \(\gamma\) raises the realized hit rate from 33\% to 84\% and correspondingly increases orthogonalization-FLOPs savings from 13.6\% to 65.0\% relative to GramMuon. For GPT-2 Small, CacheMuon saves 28\% of orthogonalization FLOPs relative to GramMuon and 43\% relative to PolarExpMuon, while for ResNet-18 it saves 15\% and 51\%, respectively. The overall hit rate is 45\% for GPT-2 Small and 80\% for ResNet-18.

\subsection{Experimental Setup and Diagnostic Logging}
\label{app:experimental-setup-and-logging}

The following paragraphs document the full experimental setup for the runs reported in \Cref{app:additional-experimental-results} and in the main text, including optimizer configurations, logging procedures, model and data details, and learning-rate schedules.

\paragraph{GPT-2 Large / OpenWebText.}
The main GPT-2 Large experiment (\Cref{fig:gpt2large-quality-flops}) uses a GPT-2 Large model trained on pre-tokenized OpenWebText. The model comprises 36 transformer blocks with 20 attention heads, embedding dimension 1280, context length 1024, and dropout 0.1. We create a small held-out validation split from OpenWebText, pre-tokenize both splits offline with the GPT-2 tokenizer, and sample training batches as random contiguous length-1024 windows from the resulting token streams. Hyperparameter tuning is conducted only for PolarExpMuon using a proxy run of 382 optimization steps on 8 GPUs with per-GPU batch size 4 and gradient accumulation factor 8, corresponding to a global batch of 256 sequences and approximately \(10^8\) tokens. Validation during tuning is performed every 25 steps using 25 batches, and hyperparameters are selected by lowest validation loss on this proxy run. The resulting base hyperparameter configuration is then reused without separate re-tuning for GramMuon and CacheMuon in the final 3,814-step training runs, corresponding to approximately \(10^9\) tokens; validation in these final runs is performed every 100 steps using 25 batches. For CacheMuon, we vary only the residual threshold \(\gamma\), using \(\gamma \in \{2,5,10,15\}\), so that the sweep isolates the quality--efficiency tradeoff induced by temporal reuse rather than re-tuning the optimizer. \Cref{fig:gpt2large-quality-flops} reports both validation loss and orthogonalization FLOPs savings for this sweep, and \Cref{tab:appendix_best_hparams_large} lists the resulting base configuration.

\paragraph{GPT-2 Large diagnostic logging.}
The residual-gate diagnostics in \Cref{fig:gpt2large-diagnostics} come from a 100M-token GPT-2 Large diagnostic run with \(\gamma=5\) using the same hyperparameters as the main experiment. We aggregate all Muon matrices logs for each optimization step. At each cache probe at step \(t\), we log the cached candidate \(\hat D_t^{\mathrm{cand}} = Q_{a_t}X_t\), the corresponding fresh direction \(\tilde D_t\) on the same normalized momentum \(X_t\), the cached-candidate residual
\[
    \frac{\|\hat D_t^{\mathrm{cand}}(\hat D_t^{\mathrm{cand}})^\top-I\|_F}{\sqrt n},
\]
and the cosine similarity between \(\mathrm{vec}(\hat D_t^{\mathrm{cand}})\) and \(\mathrm{vec}(\tilde D_t)\). The binned residual--cosine diagnostic uses all non-initial cache probes with residual in \([0,50]\), including both accepted and rejected candidates. The accept/reject diagnostic uses all probes after the initial cache-seeding step and plots a centered 15-sample rolling mean.

\paragraph{GPT-2 Small / OpenWebText.}
The GPT-2 Small / OpenWebText results in \Cref{tab:additional-run-flops} and \Cref{fig:additional-end-to-end} use a GPT-2 Small model trained on pre-tokenized OpenWebText. The model comprises 12 transformer blocks with 12 attention heads, embedding dimension 768, and context length 1024. Dropout is disabled and bias terms are omitted. Hyperparameter tuning is conducted on PolarExpMuon with a proxy run of 3,052 optimization steps, batch size 8, gradient accumulation factor 4, and context length 1024, corresponding to approximately \(10^8\) tokens. The same optimal hyperparameter configuration is used by GramMuon and CacheMuon for the final training run using 30,520 steps, corresponding to approximately \(10^9\) tokens. Additional proxy sweeps conducted independently for GramMuon and CacheMuon yielded the same optimal configuration as the PolarExpMuon sweep. Validation is performed every 100 steps using 25 batches. We create a small held-out validation split from OpenWebText, pre-tokenize both splits offline with the GPT-2 tokenizer, and sample training batches as random contiguous length-1024 windows from the resulting token streams. Hyperparameters are selected by lowest validation loss on the 100M-token proxy run. Unlike GPT-2 Large, we only use a single GPU for GPT-2 Small runs, so the effective batch size is different. 

\paragraph{ResNet-18 / CIFAR-10.}
The vision experiments (\Cref{fig:additional-end-to-end}, right) and the corresponding ResNet-18 / CIFAR-10 results in \Cref{tab:additional-run-flops} use a standard ResNet-18 architecture trained on CIFAR-10 for 100 epochs with batch size 128 and random seed 42. Training images are augmented using random \(32 \times 32\) crops with padding 4 and random horizontal flips, followed by tensor conversion and channel-wise normalization with mean \((0.4914, 0.4822, 0.4465)\) and standard deviation \((0.2023, 0.1994, 0.2010)\). Evaluation is performed on the test split using tensor conversion and the same normalization, with batch size 100. Hyperparameters are selected by the best test accuracy over the 100-epoch run. Within each task, all methods share the same training protocol. Similar to GPT-2 Small, ResNet-18 runs are conducted on a single GPU.

\paragraph{Hardware.}
The experiments for GPT-2 Large runs were conducted on 8 \texttt{NVIDIA A100-SXM4-80GB} GPUs, while the GPT-2 Small and ResNet-18 runs on a single \texttt{NVIDIA A100-SXM4-40GB} GPU. All reported GPT-2 runs use \texttt{bfloat16} arithmetic.

\paragraph{Learning-rate schedules.}
For ResNet-18 / CIFAR-10, we use cosine annealing over the full training horizon via \texttt{CosineAnnealingLR} with \(T_{\max}=100\), stepping the scheduler once per epoch. No separate warmup is used in the vision experiments. For GPT-2 / OpenWebText, the reported runs use the \texttt{plateau\_linear} schedule with \texttt{constant\_lr\_ratio}\(=0.4\) and \texttt{min\_lr\_ratio}\(=0.05\): the learning rate remains at its initial value for the first 40\% of optimization steps and then decays linearly to 5\% of the initial value by the final step.

\subsection{Hyperparameter Search Grids}
\label{app:hparam-grid}

The categorical search grid is summarized in \Cref{tab:appendix_hparam_grid}. The same grid is used for the GPT-2 Large, GPT-2 Small, and ResNet-18 sweeps; the Adam entries correspond to the Adam sub-optimizer used for non-matrix parameters. On GPT-2 Large, we do not separately tune GramMuon or CacheMuon: both reuse the base hyperparameters selected by the PolarExpMuon proxy sweep, and CacheMuon varies only \(\gamma\in\{2,5,10,15\}\) in \Cref{fig:gpt2large-quality-flops}. For the reported CacheMuon runs on GPT-2 Small / OpenWebText and ResNet-18 / CIFAR-10, we fix \(\gamma=5\) and \(\gamma=20\), respectively.

\begin{table}[htbp]
\centering
\small
\renewcommand{\arraystretch}{1.12}
\setlength{\tabcolsep}{5pt}
\begin{tabular}{@{}>{\raggedright\arraybackslash}p{0.34\linewidth}>{\raggedright\arraybackslash}p{0.60\linewidth}@{}}
\toprule
Hyperparameter & Candidate values \\
\midrule
Muon lr & \(10^{-4}, 10^{-3}, 10^{-2}, 2 \times 10^{-2}\) \\
Muon momentum & \(0.9, 0.95, 0.99\) \\
Muon weight decay & \(10^{-4}\) \\
Newton--Schulz steps & \(5\) \\
Adam lr & \(10^{-4}, 10^{-3}, 2 \times 10^{-3}\) \\
Adam betas & \((0.9, 0.999)\) \\
Adam weight decay & \(10^{-4}\) \\
\bottomrule
\end{tabular}
\vspace{4pt}
\caption{Hyperparameter search grid used for the GPT-2 Large, GPT-2 Small, and ResNet-18 experiments. The shared grid contains 36 candidate configurations.}
\label{tab:appendix_hparam_grid}
\end{table}

\subsection{Configurations Used in Reported Runs}
\label{app:reported-configurations}

The configurations used for the reported GPT-2 Large, GPT-2 Small, and ResNet-18 runs are summarized in \Cref{tab:appendix_best_hparams_large,tab:appendix_best_hparams_vision,tab:appendix_best_hparams_language}. In all three tables, the Adam rows correspond to the Adam sub-optimizer applied to non-matrix parameters. For both GPT-2 models, GramMuon and CacheMuon reuse the base hyperparameters selected by the corresponding PolarExpMuon proxy sweep; on GPT-2 Large, CacheMuon then varies only the threshold \(\gamma\) to trace the quality--efficiency frontier. 

\begin{table}[htbp]
\centering
\footnotesize
\renewcommand{\arraystretch}{1.12}
\setlength{\tabcolsep}{4pt}
\begin{tabular}{@{}>{\raggedright\arraybackslash}p{0.28\linewidth}*{3}{>{\centering\arraybackslash}p{0.17\linewidth}}@{}}
\toprule
Hyperparameter & CacheMuon & GramMuon & PolarExpMuon \\
\midrule
Muon lr & \(0.001\) & \(0.001\) & \(0.001\) \\
Muon momentum & \(0.95\) & \(0.95\) & \(0.95\) \\
Muon weight decay & \(10^{-4}\) & \(10^{-4}\) & \(10^{-4}\) \\
Newton--Schulz steps & \(5\) & \(5\) & \(5\) \\
\(\gamma\) & \(\{2,5,10,15\}\) & N/A & N/A \\
\addlinespace[2pt]
Adam lr & \(0.002\) & \(0.002\) & \(0.002\) \\
Adam betas & \((0.9, 0.999)\) & \((0.9, 0.999)\) & \((0.9, 0.999)\) \\
Adam weight decay & \(10^{-4}\) & \(10^{-4}\) & \(10^{-4}\) \\
\bottomrule
\end{tabular}
\vspace{4pt}
\caption{Configurations used for the plotted GPT-2 Large / OpenWebText runs. The PolarExpMuon proxy sweep selects the shared base hyperparameters, which are then reused by GramMuon and CacheMuon; CacheMuon varies only \(\gamma\).}
\label{tab:appendix_best_hparams_large}
\end{table}

\begin{table}[htbp]
\centering
\footnotesize
\renewcommand{\arraystretch}{1.12}
\setlength{\tabcolsep}{4pt}
\begin{tabular}{@{}>{\raggedright\arraybackslash}p{0.28\linewidth}*{3}{>{\centering\arraybackslash}p{0.17\linewidth}}@{}}
\toprule
Hyperparameter & CacheMuon & GramMuon & PolarExpMuon \\
\midrule
Muon lr & \(0.01\) & \(0.02\) & \(0.01\) \\
Muon momentum & \(0.9\) & \(0.95\) & \(0.9\) \\
Muon weight decay & \(10^{-4}\) & \(10^{-4}\) & \(10^{-4}\) \\
Newton--Schulz steps & \(5\) & \(5\) & \(5\) \\
\(\gamma\) & \(20\) & N/A & N/A \\
\addlinespace[2pt]
Adam lr & \(0.002\) & \(0.002\) & \(0.002\) \\
Adam betas & \((0.9, 0.999)\) & \((0.9, 0.999)\) & \((0.9, 0.999)\) \\
Adam weight decay & \(10^{-4}\) & \(10^{-4}\) & \(10^{-4}\) \\
\bottomrule
\end{tabular}
\vspace{4pt}
\caption{Configurations used for the ResNet-18 / CIFAR-10 runs.}
\label{tab:appendix_best_hparams_vision}
\end{table}

\begin{table}[htbp]
\centering
\footnotesize
\renewcommand{\arraystretch}{1.12}
\setlength{\tabcolsep}{4pt}
\begin{tabular}{@{}>{\raggedright\arraybackslash}p{0.28\linewidth}*{3}{>{\centering\arraybackslash}p{0.17\linewidth}}@{}}
\toprule
Hyperparameter & CacheMuon & GramMuon & PolarExpMuon \\
\midrule
Muon lr & \(0.001\) & \(0.001\) & \(0.001\) \\
Muon momentum & \(0.95\) & \(0.95\) & \(0.95\) \\
Muon weight decay & \(10^{-4}\) & \(10^{-4}\) & \(10^{-4}\) \\
Newton--Schulz steps & \(5\) & \(5\) & \(5\) \\
\(\gamma\) & \(5\) & N/A & N/A \\
\addlinespace[2pt]
Adam lr & \(0.002\) & \(0.002\) & \(0.002\) \\
Adam betas & \((0.9, 0.999)\) & \((0.9, 0.999)\) & \((0.9, 0.999)\) \\
Adam weight decay & \(10^{-4}\) & \(10^{-4}\) & \(10^{-4}\) \\
\bottomrule
\end{tabular}
\vspace{4pt}
\caption{Configurations used for the plotted GPT-2 Small on OpenWebText runs. All three optimizers share the same hyperparameters except \(\gamma\), which is only used by CacheMuon.}
\label{tab:appendix_best_hparams_language}
\end{table}

\subsection{Implementation Details}
\label{app:implementation-selection}

\paragraph{Orthogonalization.}
All reported Muon-family runs use \(K=5\) Newton--Schulz steps. The PolarExpMuon coefficients are taken from PolarExpress \citep{polar_express}, while the GramMuon and CacheMuon coefficients are taken from GramMuon \citep{gram_newton_schulz}. For GramMuon and CacheMuon we use the stabilized Gram Newton--Schulz restart set \(\mathcal S=\{2\}\), and CacheMuon uses the same fresh solver on refresh steps. We set \(\varepsilon=10^{-7}\) throughout. The effective coefficient triplets \((a_k,b_k,c_k)\) used in the reported runs are:

\begingroup
\centering
\scriptsize
\renewcommand{\arraystretch}{1.0}
\setlength{\tabcolsep}{3pt}
\begin{minipage}[t]{0.48\linewidth}
\centering
\textbf{PolarExpMuon}

\begin{tabular}{@{}rccc@{}}
\toprule
\(k\) & \(a_k\) & \(b_k\) & \(c_k\) \\
\midrule
1 & \(8.205160414\) & \(-22.901934987\) & \(16.460724910\) \\
2 & \(4.066395160\) & \(-2.861154087\) & \(0.518399523\) \\
3 & \(3.909594904\) & \(-2.823351735\) & \(0.525036977\) \\
4 & \(3.285564017\) & \(-2.415301960\) & \(0.485294066\) \\
5 & \(2.277873287\) & \(-1.619821765\) & \(0.398480787\) \\
\bottomrule
\end{tabular}
\end{minipage}\hfill
\begin{minipage}[t]{0.48\linewidth}
\centering
\textbf{GramMuon / CacheMuon}

\begin{tabular}{@{}rccc@{}}
\toprule
\(k\) & \(a_k\) & \(b_k\) & \(c_k\) \\
\midrule
1 & \(7.892582874\) & \(-20.383013946\) & \(13.555306149\) \\
2 & \(3.911484868\) & \(-2.546463593\) & \(0.426898832\) \\
3 & \(3.760657956\) & \(-2.512819018\) & \(0.432364735\) \\
4 & \(3.160399674\) & \(-2.149649519\) & \(0.399636691\) \\
5 & \(2.191097162\) & \(-1.441662010\) & \(0.328146488\) \\
\bottomrule
\end{tabular}
\end{minipage}
\par
\endgroup

\paragraph{Parameter groups.}
Across both tasks, trainable tensors with at least two dimensions are assigned to the Muon update, and the remaining trainable tensors are assigned to the Adam sub-optimizer. For GPT-2, embedding matrices are excluded from the Muon group, so Muon updates are applied only to the attention and MLP weight matrices, while embeddings, biases, and one-dimensional parameters remain in the Adam group.

%% file: sections/appendix3.tex
\section{Proofs for Section~\ref{sec:convergence}}
\label{app:proofs}

We begin with a basic structural fact about the normalized Gram matrices.

\begin{lemma}
\label{lem:gram-domain}
For every \(t\),
\[
A_t \in \mathcal D
=
\{A\in\mathbb R^{n\times n}: A=A^\top,\ A\succeq 0,\ \|A\|_{\mathrm{sp}}\le 1\}.
\]
\end{lemma}

\begin{proof}
By definition,
\[
A_t = X_t X_t^\top.
\]
Hence \(A_t\) is symmetric:
\[
A_t^\top = (X_t X_t^\top)^\top = X_t X_t^\top = A_t.
\]

It is also positive semidefinite, since for any vector \(v\),
\[
v^\top A_t v = v^\top X_t X_t^\top v = \|X_t^\top v\|_2^2 \ge 0.
\]

Finally,
\[
\|X_t\|_{\mathrm{sp}} \le \|X_t\|_F
=
\frac{\|M_t\|_F}{\|M_t\|_F+\varepsilon}
\le 1.
\]
Therefore,
\[
\|A_t\|_{\mathrm{sp}}
=
\|X_t X_t^\top\|_{\mathrm{sp}}
=
\|X_t\|_{\mathrm{sp}}^2
\le 1.
\]

Thus \(A_t\) is symmetric, positive semidefinite, and has spectral norm at most \(1\), so \(A_t \in \mathcal D\).
\end{proof}

\begin{lemma}
\label{lem:polynomial-map}
For fixed coefficients \(\{(a_k,b_k,c_k)\}_{k=1}^K\) and a fixed restart set \(\mathcal S\), the fresh stabilized Gram--Newton--Schulz solver induces a matrix polynomial map
\[
A \longmapsto \mathcal Q(A)
\qquad \text{on } \mathcal D.
\]
Equivalently, there exists a scalar polynomial \(\psi\) such that
\[
\mathcal Q(A)=\psi(A)
\qquad \text{for all } A\in\mathcal D.
\]
\end{lemma}

\begin{proof}
Fix a matrix \(A \in \mathcal D\), and let \(X\) be any normalized matrix satisfying
\[
A = X X^\top.
\]
We now run the fresh stabilized Gram--Newton--Schulz routine on \(X\), using the fixed coefficients and fixed restart set from Section~\ref{sec:method}.

For the proof, it is convenient to track the \emph{cumulative} left transform. Let \(S_k\) denote the total left transform after \(k\) inner iterations, so that the current iterate can be written as
\[
Y_k = S_k X.
\]
We also let \(R_k\) denote the current Gram matrix after \(k\) inner iterations.

We claim that for every \(k=0,1,\dots,K\),
\begin{equation}
\label{eq:poly-invariant}
S_k \text{ is a polynomial in } A,
\qquad
Y_k = S_k X,
\qquad
R_k = S_k A S_k.
\end{equation}
We prove \eqref{eq:poly-invariant} by induction on \(k\).

\paragraph{Base case.}
At initialization,
\[
S_0 = I,
\qquad
Y_0 = X,
\qquad
R_0 = A.
\]
The identity matrix is the constant polynomial in \(A\), and clearly
\[
Y_0 = S_0 X,
\qquad
R_0 = S_0 A S_0.
\]
So \eqref{eq:poly-invariant} holds for \(k=0\).

\paragraph{First iteration.}
At the first inner iteration,
\[
P_1 = a_1 I + b_1 R_0 + c_1 R_0^2
      = a_1 I + b_1 A + c_1 A^2,
\]
so \(P_1\) is a polynomial in \(A\). The cumulative transform becomes
\[
S_1 = P_1 S_0 = P_1,
\]
and is therefore also a polynomial in \(A\). The current iterate is
\[
Y_1 = S_1 X.
\]

Now consider the Gram matrix.

If \(1 \notin \mathcal S\), the algorithm propagates the Gram matrix algebraically:
\[
R_1 = P_1 R_0 P_1 = P_1 A P_1 = S_1 A S_1.
\]

If \(1 \in \mathcal S\), the algorithm performs a restart and recomputes the Gram matrix from the refreshed iterate:
\[
R_1 = Y_1 Y_1^\top = (S_1 X)(S_1 X)^\top = S_1 A S_1^\top.
\]
Because \(S_1\) is a polynomial in the symmetric matrix \(A\), it is itself symmetric, so \(S_1^\top = S_1\). Hence
\[
R_1 = S_1 A S_1.
\]

Thus \eqref{eq:poly-invariant} holds for \(k=1\).

\paragraph{Induction step.}
Assume \eqref{eq:poly-invariant} holds at step \(k-1\), namely,
\[
S_{k-1} \text{ is a polynomial in } A,
\qquad
Y_{k-1} = S_{k-1} X,
\qquad
R_{k-1} = S_{k-1} A S_{k-1}.
\]
We show that it also holds at step \(k\).

The fresh stabilized Gram--Newton--Schulz update forms
\[
P_k = a_k I + b_k R_{k-1} + c_k R_{k-1}^2.
\]
Since \(R_{k-1}\) is a polynomial in \(A\), \(P_k\) is also a polynomial in \(A\).

The cumulative transform is updated by
\[
S_k = P_k S_{k-1}.
\]
A product of two polynomials in the same matrix \(A\) is again a polynomial in \(A\), so \(S_k\) is a polynomial in \(A\).

The current iterate is
\[
Y_k = S_k X
\]
by definition.

It remains to verify the Gram matrix identity.

If \(k \notin \mathcal S\), the Gram matrix is propagated algebraically:
\[
R_k = P_k R_{k-1} P_k.
\]
Using the induction hypothesis,
\[
R_{k-1} = S_{k-1} A S_{k-1},
\]
so
\[
R_k = P_k S_{k-1} A S_{k-1} P_k.
\]
Since both \(P_k\) and \(S_{k-1}\) are polynomials in the same matrix \(A\), they commute. Therefore
\[
P_k S_{k-1} = S_{k-1} P_k = S_k,
\]
and hence
\[
R_k = S_k A S_k.
\]

If \(k \in \mathcal S\), the algorithm performs a restart and recomputes
\[
R_k = Y_k Y_k^\top = (S_k X)(S_k X)^\top = S_k A S_k^\top.
\]
Again, because \(S_k\) is a polynomial in the symmetric matrix \(A\), it is symmetric, so \(S_k^\top = S_k\). Thus
\[
R_k = S_k A S_k.
\]

This completes the induction.

In particular, after the final inner iteration, the fresh solver returns the final cumulative transform
\[
\mathcal Q(A) = S_K.
\]
Since \(S_K\) is a polynomial in \(A\), there exists a scalar polynomial \(\psi\) such that
\[
\mathcal Q(A)=\psi(A).
\]
This proves the lemma.
\end{proof}

\begin{lemma}
\label{lem:lipschitz-map}
There exists a constant \(L_Q>0\) such that, for all \(A,B\in\mathcal D\),
\[
\|\mathcal Q(A)-\mathcal Q(B)\|_{\mathrm{sp}}
\le
L_Q \|A-B\|_{\mathrm{sp}}.
\]
Moreover, if
\[
\mathcal Q(A)=\psi(A)=\sum_{j=0}^d \alpha_j A^j,
\]
then one admissible choice is
\[
L_Q = \sum_{j=1}^d j |\alpha_j|.
\]
\end{lemma}

\begin{proof}
By Lemma~\ref{lem:polynomial-map}, there exists a scalar polynomial
\[
\psi(\lambda)=\sum_{j=0}^d \alpha_j \lambda^j
\]
such that
\[
\mathcal Q(A)=\psi(A)=\sum_{j=0}^d \alpha_j A^j
\qquad \text{for all } A\in\mathcal D.
\]
Therefore
\[
\mathcal Q(A)-\mathcal Q(B)
=
\sum_{j=1}^d \alpha_j (A^j-B^j),
\]
since the constant term cancels.

For each \(j\ge 1\), use the telescoping identity
\[
A^j-B^j
=
\sum_{\ell=0}^{j-1} A^\ell (A-B) B^{j-1-\ell}.
\]
Taking spectral norms gives
\[
\|A^j-B^j\|_{\mathrm{sp}}
\le
\sum_{\ell=0}^{j-1}
\|A^\ell\|_{\mathrm{sp}}
\|A-B\|_{\mathrm{sp}}
\|B^{j-1-\ell}\|_{\mathrm{sp}}.
\]

Because \(A,B\in\mathcal D\), we have
\[
\|A\|_{\mathrm{sp}}\le 1,
\qquad
\|B\|_{\mathrm{sp}}\le 1.
\]
Hence
\[
\|A^\ell\|_{\mathrm{sp}}\le 1,
\qquad
\|B^{j-1-\ell}\|_{\mathrm{sp}}\le 1.
\]
So each term in the sum is bounded by \(\|A-B\|_{\mathrm{sp}}\), and there are \(j\) such terms. Therefore
\[
\|A^j-B^j\|_{\mathrm{sp}}
\le
j\,\|A-B\|_{\mathrm{sp}}.
\]

Substituting this into the polynomial difference yields
\[
\|\mathcal Q(A)-\mathcal Q(B)\|_{\mathrm{sp}}
\le
\sum_{j=1}^d |\alpha_j|\,\|A^j-B^j\|_{\mathrm{sp}}
\le
\sum_{j=1}^d j|\alpha_j|\,\|A-B\|_{\mathrm{sp}}.
\]
Thus
\[
\|\mathcal Q(A)-\mathcal Q(B)\|_{\mathrm{sp}}
\le
L_Q \|A-B\|_{\mathrm{sp}}
\]
with
\[
L_Q = \sum_{j=1}^d j|\alpha_j|.
\]
This proves the lemma.
\end{proof}

\begin{proof}[Proof of Proposition~\ref{prop:cache-inexact}]
By Lemma~\ref{lem:gram-domain}, \(A_t \in \mathcal D\) for all \(t\). By Lemma~\ref{lem:polynomial-map}, the fresh stabilized Gram--Newton--Schulz solver induces a matrix polynomial map
\[
A \mapsto \mathcal Q(A)
\qquad \text{on } \mathcal D.
\]
By Lemma~\ref{lem:lipschitz-map}, there exists a constant \(L_Q>0\) such that
\[
\|\mathcal Q(A)-\mathcal Q(B)\|_{\mathrm{sp}}
\le
L_Q \|A-B\|_{\mathrm{sp}},
\qquad
A,B\in\mathcal D.
\]

Now recall that
\[
\tilde D_t = \mathcal Q(A_t)X_t,
\qquad
\hat D_t = \mathcal Q(A_{a_t})X_t.
\]
Therefore
\[
\hat D_t-\tilde D_t
=
\bigl(\mathcal Q(A_{a_t})-\mathcal Q(A_t)\bigr)X_t.
\]
Taking spectral norms gives
\[
\|\hat D_t-\tilde D_t\|_{\mathrm{sp}}
\le
\|\mathcal Q(A_{a_t})-\mathcal Q(A_t)\|_{\mathrm{sp}}\,\|X_t\|_{\mathrm{sp}}.
\]
Since
\[
X_t = \frac{M_t}{\|M_t\|_F+\varepsilon},
\]
we have
\[
\|X_t\|_{\mathrm{sp}} \le \|X_t\|_F
=
\frac{\|M_t\|_F}{\|M_t\|_F+\varepsilon}
\le 1.
\]
Hence
\[
\|\hat D_t-\tilde D_t\|_{\mathrm{sp}}
\le
\|\mathcal Q(A_{a_t})-\mathcal Q(A_t)\|_{\mathrm{sp}}
\le
L_Q \|A_t-A_{a_t}\|_{\mathrm{sp}}
=
L_Q \tau_t.
\]

Finally, by triangle inequality and Assumption~\textbf{(A3)},
\[
\|\hat D_t-D_t\|_{\mathrm{sp}}
\le
\|\hat D_t-\tilde D_t\|_{\mathrm{sp}}
+
\|\tilde D_t-D_t\|_{\mathrm{sp}}
\le
L_Q \tau_t + \varepsilon_{\mathrm{ref}}
=
\delta_t.
\]
This proves the proposition.
\end{proof}

We now prove Theorem~\ref{thm:main}. We begin with a basic alignment property of the exact Muon direction.

\begin{lemma}
\label{lem:exact-polar-alignment}
For every \(t\),
\[
\langle D_t, M_t\rangle = \|M_t\|_{\mathrm{tr}},
\qquad
\|D_t\|_{\mathrm{sp}} = 1.
\]
\end{lemma}

\begin{proof}
Let
\[
M_t = U_t \Sigma_t V_t^\top
\]
be the reduced singular value decomposition of \(M_t\). By definition,
\[
D_t = \polar(M_t) = U_t V_t^\top.
\]
Therefore,
\[
\langle D_t, M_t\rangle
=
\operatorname{tr}(D_t^\top M_t)
=
\operatorname{tr}(V_t U_t^\top U_t \Sigma_t V_t^\top)
=
\operatorname{tr}(V_t \Sigma_t V_t^\top).
\]
Using cyclicity of the trace,
\[
\operatorname{tr}(V_t \Sigma_t V_t^\top)
=
\operatorname{tr}(\Sigma_t V_t^\top V_t)
=
\operatorname{tr}(\Sigma_t)
=
\|M_t\|_{\mathrm{tr}}.
\]
This proves the first identity.

For the second claim, \(D_t = U_t V_t^\top\) is a partial isometry whose nonzero singular values are all equal to \(1\). Hence
\[
\|D_t\|_{\mathrm{sp}} = 1.
\]
\end{proof}

\begin{lemma}
\label{lem:approx-alignment}
For every \(t\),
\[
\|\hat D_t\|_{\mathrm{sp}} \le 1+\delta_t
\]
and
\[
\langle \hat D_t,\nabla F(W_t)\rangle
\ge
(1-\delta_t)\|\nabla F(W_t)\|_{\mathrm{tr}} - 2e_t.
\]
\end{lemma}

\begin{proof}
By Proposition~\ref{prop:cache-inexact},
\[
\|\hat D_t-D_t\|_{\mathrm{sp}} \le \delta_t.
\]
Using Lemma~\ref{lem:exact-polar-alignment},
\[
\|\hat D_t\|_{\mathrm{sp}}
\le
\|D_t\|_{\mathrm{sp}} + \|\hat D_t-D_t\|_{\mathrm{sp}}
\le
1+\delta_t.
\]

Next, write
\[
\langle D_t,\nabla F(W_t)\rangle
=
\langle D_t,M_t\rangle + \langle D_t,\nabla F(W_t)-M_t\rangle.
\]
By Lemma~\ref{lem:exact-polar-alignment},
\[
\langle D_t,M_t\rangle = \|M_t\|_{\mathrm{tr}}.
\]
Using duality between the spectral and trace norms,
\[
\langle D_t,\nabla F(W_t)-M_t\rangle
\ge
-\|D_t\|_{\mathrm{sp}}\,\|\nabla F(W_t)-M_t\|_{\mathrm{tr}}
=
-e_t.
\]
Hence
\[
\langle D_t,\nabla F(W_t)\rangle
\ge
\|M_t\|_{\mathrm{tr}} - e_t.
\]

Also, by the triangle inequality for the trace norm,
\[
\|M_t\|_{\mathrm{tr}}
\ge
\|\nabla F(W_t)\|_{\mathrm{tr}} - \|M_t-\nabla F(W_t)\|_{\mathrm{tr}}
=
\|\nabla F(W_t)\|_{\mathrm{tr}} - e_t.
\]
Combining the last two displays gives
\[
\langle D_t,\nabla F(W_t)\rangle
\ge
\|\nabla F(W_t)\|_{\mathrm{tr}} - 2e_t.
\]

Finally,
\[
\langle \hat D_t,\nabla F(W_t)\rangle
=
\langle D_t,\nabla F(W_t)\rangle
+
\langle \hat D_t-D_t,\nabla F(W_t)\rangle.
\]
Again by duality,
\[
\langle \hat D_t-D_t,\nabla F(W_t)\rangle
\ge
-\|\hat D_t-D_t\|_{\mathrm{sp}}\,\|\nabla F(W_t)\|_{\mathrm{tr}}
\ge
-\delta_t \|\nabla F(W_t)\|_{\mathrm{tr}}.
\]
Therefore
\[
\langle \hat D_t,\nabla F(W_t)\rangle
\ge
(1-\delta_t)\|\nabla F(W_t)\|_{\mathrm{tr}} - 2e_t.
\]
\end{proof}

\begin{lemma}
\label{lem:one-step-descent}
Let
\[
\bar\delta := L_Q \bar\tau + \varepsilon_{\mathrm{ref}}.
\]
Then for every \(t\),
\[
F(W_{t+1})
\le
F(W_t)
-
\eta(1-\bar\delta)\|\nabla F(W_t)\|_{\mathrm{tr}}
+
2\eta e_t
+
\frac{L\eta^2}{2}(1+\bar\delta)^2.
\]
\end{lemma}

\begin{proof}
By the update rule of CacheMuon,
\[
W_{t+1} = W_t - \eta \hat D_t.
\]
Applying the smoothness inequality from Assumption~\textbf{(A1)} with \(X=W_t\) and \(Y=W_{t+1}\), we get
\[
F(W_{t+1})
\le
F(W_t)
+
\langle \nabla F(W_t), W_{t+1}-W_t\rangle
+
\frac{L}{2}\|W_{t+1}-W_t\|_{\mathrm{sp}}^2.
\]
Since
\[
W_{t+1}-W_t = -\eta \hat D_t,
\]
this becomes
\[
F(W_{t+1})
\le
F(W_t)
-
\eta \langle \nabla F(W_t), \hat D_t\rangle
+
\frac{L\eta^2}{2}\|\hat D_t\|_{\mathrm{sp}}^2.
\]

By Lemma~\ref{lem:approx-alignment},
\[
\langle \hat D_t,\nabla F(W_t)\rangle
\ge
(1-\delta_t)\|\nabla F(W_t)\|_{\mathrm{tr}} - 2e_t
\]
and
\[
\|\hat D_t\|_{\mathrm{sp}} \le 1+\delta_t.
\]
Substituting these inequalities yields
\[
F(W_{t+1})
\le
F(W_t)
-
\eta\bigl((1-\delta_t)\|\nabla F(W_t)\|_{\mathrm{tr}} - 2e_t\bigr)
+
\frac{L\eta^2}{2}(1+\delta_t)^2.
\]
That is,
\[
F(W_{t+1})
\le
F(W_t)
-
\eta(1-\delta_t)\|\nabla F(W_t)\|_{\mathrm{tr}}
+
2\eta e_t
+
\frac{L\eta^2}{2}(1+\delta_t)^2.
\]

Finally, by the definition of \(\delta_t\) in Proposition~\ref{prop:cache-inexact},
\[
\delta_t = L_Q \tau_t + \varepsilon_{\mathrm{ref}}
\le
L_Q \bar\tau + \varepsilon_{\mathrm{ref}}
=
\bar\delta.
\]
Hence
\[
1-\delta_t \ge 1-\bar\delta
\qquad\text{and}\qquad
(1+\delta_t)^2 \le (1+\bar\delta)^2.
\]
Using these bounds gives
\[
F(W_{t+1})
\le
F(W_t)
-
\eta(1-\bar\delta)\|\nabla F(W_t)\|_{\mathrm{tr}}
+
2\eta e_t
+
\frac{L\eta^2}{2}(1+\bar\delta)^2.
\]
This proves the claim.
\end{proof}

\begin{lemma}
\label{lem:tracking}
Let
\[
\bar\delta := L_Q \bar\tau + \varepsilon_{\mathrm{ref}}.
\]
Then for every \(t \ge 1\),
\[
\mathbb E[e_t]
\le
\beta^{t-1}\mathbb E[e_1]
+
\frac{\beta L\eta(1+\bar\delta)}{1-\beta}
+
\rho \sigma \sqrt{1-\beta}.
\]
\end{lemma}

\begin{proof}
Define
\[
E_t := M_t - \nabla F(W_t),
\qquad
e_t = \|E_t\|_{\mathrm{tr}}.
\]

For \(t \ge 2\), using the momentum recursion and Assumption~\textbf{(A2)},
\[
M_t = (1-\beta)g_t + \beta M_{t-1}
     = (1-\beta)\nabla F(W_t) + \beta M_{t-1} + (1-\beta)\zeta_t.
\]
Subtracting \(\nabla F(W_t)\) gives
\[
E_t
=
\beta\bigl(M_{t-1}-\nabla F(W_t)\bigr) + (1-\beta)\zeta_t.
\]
Add and subtract \(\beta \nabla F(W_{t-1})\):
\[
E_t
=
\beta\bigl(M_{t-1}-\nabla F(W_{t-1})\bigr)
+
\beta\bigl(\nabla F(W_{t-1})-\nabla F(W_t)\bigr)
+
(1-\beta)\zeta_t.
\]
Thus
\[
E_t
=
\beta E_{t-1}
+
\beta\bigl(\nabla F(W_{t-1})-\nabla F(W_t)\bigr)
+
(1-\beta)\zeta_t.
\]

Unrolling this recursion yields
\[
E_t
=
\beta^{t-1} E_1
+
\beta \sum_{i=2}^{t} \beta^{t-i}\bigl(\nabla F(W_{i-1})-\nabla F(W_i)\bigr)
+
(1-\beta)\sum_{i=2}^{t}\beta^{t-i}\zeta_i.
\]
Taking trace norms and expectations,
\begin{align*}
\mathbb E[e_t]
&\le
\beta^{t-1}\mathbb E[e_1]
+
\beta \sum_{i=2}^{t}\beta^{t-i}\,
\mathbb E\|\nabla F(W_{i-1})-\nabla F(W_i)\|_{\mathrm{tr}} \\
&\qquad
+
\mathbb E\left\|
(1-\beta)\sum_{i=2}^{t}\beta^{t-i}\zeta_i
\right\|_{\mathrm{tr}}.
\end{align*}

We bound the last two terms separately.

For the gradient-drift term, Assumption~\textbf{(A1)} gives
\[
\|\nabla F(W_{i-1})-\nabla F(W_i)\|_{\mathrm{tr}}
\le
L\|W_{i-1}-W_i\|_{\mathrm{sp}}.
\]
Since
\[
W_i - W_{i-1} = -\eta \hat D_{i-1},
\]
Lemma~\ref{lem:approx-alignment} implies
\[
\|W_i-W_{i-1}\|_{\mathrm{sp}}
=
\eta \|\hat D_{i-1}\|_{\mathrm{sp}}
\le
\eta(1+\bar\delta).
\]
Therefore
\[
\|\nabla F(W_{i-1})-\nabla F(W_i)\|_{\mathrm{tr}}
\le
L\eta(1+\bar\delta),
\]
and hence
\[
\beta \sum_{i=2}^{t}\beta^{t-i}\,
\mathbb E\|\nabla F(W_{i-1})-\nabla F(W_i)\|_{\mathrm{tr}}
\le
\beta L\eta(1+\bar\delta)\sum_{j=0}^{t-2}\beta^j
\le
\frac{\beta L\eta(1+\bar\delta)}{1-\beta}.
\]

For the noise term, let \(\rho>0\) denote a norm-equivalence constant such that
\[
\|Z\|_{\mathrm{tr}} \le \rho \|Z\|_F
\]
for all matrices \(Z\) in this finite-dimensional setting. Then
\[
\mathbb E\left\|
(1-\beta)\sum_{i=2}^{t}\beta^{t-i}\zeta_i
\right\|_{\mathrm{tr}}
\le
\rho
\left(
\mathbb E\left\|
(1-\beta)\sum_{i=2}^{t}\beta^{t-i}\zeta_i
\right\|_F^2
\right)^{1/2}.
\]

Expanding the squared Frobenius norm produces diagonal terms and cross terms. For \(i<j\), conditioning on all information available before drawing \(g_j\), the vector \(\zeta_i\) is fixed while Assumption~\textbf{(A2)} gives zero conditional mean for \(\zeta_j\). Therefore
\[
\mathbb E\langle \zeta_i,\zeta_j\rangle = 0.
\]
Hence all cross terms vanish, and
\[
\mathbb E\left\|
(1-\beta)\sum_{i=2}^{t}\beta^{t-i}\zeta_i
\right\|_F^2
\le
(1-\beta)^2 \sigma^2 \sum_{j=0}^{t-2}\beta^{2j}.
\]
Therefore
\[
\mathbb E\left\|
(1-\beta)\sum_{i=2}^{t}\beta^{t-i}\zeta_i
\right\|_{\mathrm{tr}}
\le
\rho \sigma (1-\beta)\left(\sum_{j=0}^{t-2}\beta^{2j}\right)^{1/2}.
\]
Using
\[
\sum_{j=0}^{t-2}\beta^{2j} \le \frac{1}{1-\beta^2},
\]
we obtain
\[
\rho \sigma (1-\beta)\left(\sum_{j=0}^{t-2}\beta^{2j}\right)^{1/2}
\le
\rho \sigma \sqrt{\frac{1-\beta}{1+\beta}}
\le
\rho \sigma \sqrt{1-\beta}.
\]

Combining the bounds proves the lemma.
\end{proof}

\begin{proof}[Proof of Theorem~\ref{thm:main}]
By Lemma~\ref{lem:one-step-descent}, for every \(t\),
\[
F(W_{t+1})
\le
F(W_t)
-
\eta(1-\bar\delta)\|\nabla F(W_t)\|_{\mathrm{tr}}
+
2\eta e_t
+
\frac{L\eta^2}{2}(1+\bar\delta)^2.
\]
Rearranging gives
\[
\eta(1-\bar\delta)\|\nabla F(W_t)\|_{\mathrm{tr}}
\le
F(W_t)-F(W_{t+1})
+
2\eta e_t
+
\frac{L\eta^2}{2}(1+\bar\delta)^2.
\]

Taking expectations and summing from \(t=1\) to \(N\), we obtain
\[
\eta(1-\bar\delta)\sum_{t=1}^{N}\mathbb E\|\nabla F(W_t)\|_{\mathrm{tr}}
\le
\sum_{t=1}^{N}\mathbb E\!\left[F(W_t)-F(W_{t+1})\right]
+
2\eta\sum_{t=1}^{N}\mathbb E[e_t]
+
\frac{LN\eta^2}{2}(1+\bar\delta)^2.
\]

The objective terms telescope:
\[
\sum_{t=1}^{N}\mathbb E\!\left[F(W_t)-F(W_{t+1})\right]
=
\mathbb E\!\left[F(W_1)-F(W_{N+1})\right]
\le
F(W_1)-F_*
=
\Delta_1.
\]
Therefore,
\[
\eta(1-\bar\delta)\sum_{t=1}^{N}\mathbb E\|\nabla F(W_t)\|_{\mathrm{tr}}
\le
\Delta_1
+
2\eta\sum_{t=1}^{N}\mathbb E[e_t]
+
\frac{LN\eta^2}{2}(1+\bar\delta)^2.
\]

Now apply Lemma~\ref{lem:tracking}:
\[
\mathbb E[e_t]
\le
\beta^{t-1}\mathbb E[e_1]
+
\frac{\beta L\eta(1+\bar\delta)}{1-\beta}
+
\rho \sigma \sqrt{1-\beta}.
\]
Summing this bound over \(t=1,\dots,N\) yields
\[
\sum_{t=1}^{N}\mathbb E[e_t]
\le
\mathbb E[e_1]\sum_{t=1}^{N}\beta^{t-1}
+
N\frac{\beta L\eta(1+\bar\delta)}{1-\beta}
+
N\rho \sigma \sqrt{1-\beta}.
\]
Since
\[
\sum_{t=1}^{N}\beta^{t-1}\le \frac{1}{1-\beta},
\]
we obtain
\[
\sum_{t=1}^{N}\mathbb E[e_t]
\le
\frac{\mathbb E[e_1]}{1-\beta}
+
N\frac{\beta L\eta(1+\bar\delta)}{1-\beta}
+
N\rho \sigma \sqrt{1-\beta}.
\]

Substituting this into the previous inequality gives
\[
\eta(1-\bar\delta)\sum_{t=1}^{N}\mathbb E\|\nabla F(W_t)\|_{\mathrm{tr}}
\le
\Delta_1
+
2\eta\left(
\frac{\mathbb E[e_1]}{1-\beta}
+
N\frac{\beta L\eta(1+\bar\delta)}{1-\beta}
+
N\rho \sigma \sqrt{1-\beta}
\right)
+
\frac{LN\eta^2}{2}(1+\bar\delta)^2.
\]

Finally, dividing both sides by \(N\eta(1-\bar\delta)\) yields
\[
\frac{1}{N}\sum_{t=1}^{N}\mathbb E\|\nabla F(W_t)\|_{\mathrm{tr}}
\le
\frac{1}{1-\bar\delta}
\left[
\frac{\Delta_1}{N\eta}
+
\frac{2\mathbb E[e_1]}{N(1-\beta)}
+
\frac{2\beta L\eta(1+\bar\delta)}{1-\beta}
+
2\rho \sigma \sqrt{1-\beta}
+
\frac{L\eta}{2}(1+\bar\delta)^2
\right].
\]
This proves the first part of the theorem.

For the stated rate, choose
\[
\eta = \Theta(N^{-3/4}),
\qquad
1-\beta = \Theta(N^{-1/2}).
\]
Then the five terms in the bracket scale as
\[
\frac{\Delta_1}{N\eta} = O(N^{-1/4}),
\qquad
\frac{2\mathbb E[e_1]}{N(1-\beta)} = O(N^{-1/2}),
\]
\[
\frac{2\beta L\eta(1+\bar\delta)}{1-\beta} = O(N^{-1/4}),
\qquad
2\rho \sigma \sqrt{1-\beta} = O(N^{-1/4}),
\]
and
\[
\frac{L\eta}{2}(1+\bar\delta)^2 = O(N^{-3/4}).
\]
The dominant terms are therefore of order \(N^{-1/4}\), and hence
\[
\frac{1}{N}\sum_{t=1}^{N}\mathbb E\|\nabla F(W_t)\|_{\mathrm{tr}}
=
O\!\left(\frac{1}{(1-\bar\delta)N^{1/4}}\right).
\]
This completes the proof.
\end{proof}